\documentclass[11pt]{article}

\usepackage[margin=1in]{geometry}

\newcommand{\eg}{\emph{e.g.}\@\xspace}

\usepackage{graphicx}
\usepackage{booktabs}

\usepackage{amsmath,amssymb,amsthm,mathtools}
\usepackage{xspace}
\usepackage{xcolor}
\usepackage{hyperref}
\usepackage[capitalize]{cleveref}   
\usepackage{enumitem}
\usepackage{microtype}
\usepackage{wrapfig} 
\usepackage{caption}

\newtheorem{theorem}{Theorem}

\newtheorem{proposition}{Proposition}

\newcommand{\R}{\mathbb{R}}
\newcommand{\SO}{\mathrm{SO}}
\newcommand{\SE}{\mathrm{SE}}

\newcommand{\Sym}{\mathrm{Sym}}
\newcommand{\SPD}{\mathrm{SPD}}

\newcommand{\tr}{\mathrm{tr}}
\newcommand{\Ad}{\mathrm{Ad}}

\newcommand{\End}{\operatorname{End}}

\newcommand{\Hom}{\operatorname{Hom}}

\newcommand{\hatop}[1]{[#1]_\times}

\begin{document}

\title{\textbf{E3DGS: Unified Geometric--Photometric Equivariance for\\ 3D Gaussian Splatting via Color-as-Geometry Embedding}}

\author{
Chankyo Kim \quad Maani Ghaffari \\[2pt]
University of Michigan, Ann Arbor, MI, USA \\
\texttt{\{chankyo, maanigj\}@umich.edu}
}
\date{}
\maketitle

\begin{abstract}
3D Gaussian Splatting (3DGS) captures scenes by coupling explicit geometry (position, covariance) with view-dependent photometry (Spherical Harmonics). However, building $\mathrm{SE}(3)$-equivariant architectures on these primitives presents a fundamental representation bottleneck. Color has been treated as a signal rather than a geometric entity, making it nontrivial to unify symmetry across geometry and appearance as the camera frame changes. While translations are handled by relative coordinates, rotations act heterogeneously across attributes:
$\mu\mapsto R\mu$, $\Sigma\mapsto R\Sigma R^\top$, and $f_\ell\mapsto D^\ell(R)f_\ell$. This mismatch complicates strict equivariance, leading existing methods to either discard or flatten SH coefficients, thereby breaking symmetry.

We propose a unified solution rooted in representation theory: for SH degrees $\ell \le 2$, photometry is algebraically isomorphic to a rank-2 geometric tensor. We prove that the Wigner-$D$ action on these SH coefficients can be exactly reformulated as the conjugation action on $3 \times 3$ matrices. Leveraging this, we introduce the Unified Matrix Embedding, a lifting that maps all Gaussian attributes into a unified carrier space, $\mathfrak{gl}(3)$.

Building on ``Color-as-Geometry'' formulation, we present \textbf{E3DGS}, a rigid-body ($\mathrm{SE}(3)$) equivariant architecture that processes 3D Gaussians without Clebsch--Gordan tensor products. Evaluations on object vision and action-conditioned Gaussian world modeling demonstrate that our unified approach yields strong robustness under camera-frame changes and improved data efficiency.

\end{abstract}


\begin{figure*}[t]
\centering
\includegraphics[width=\textwidth]{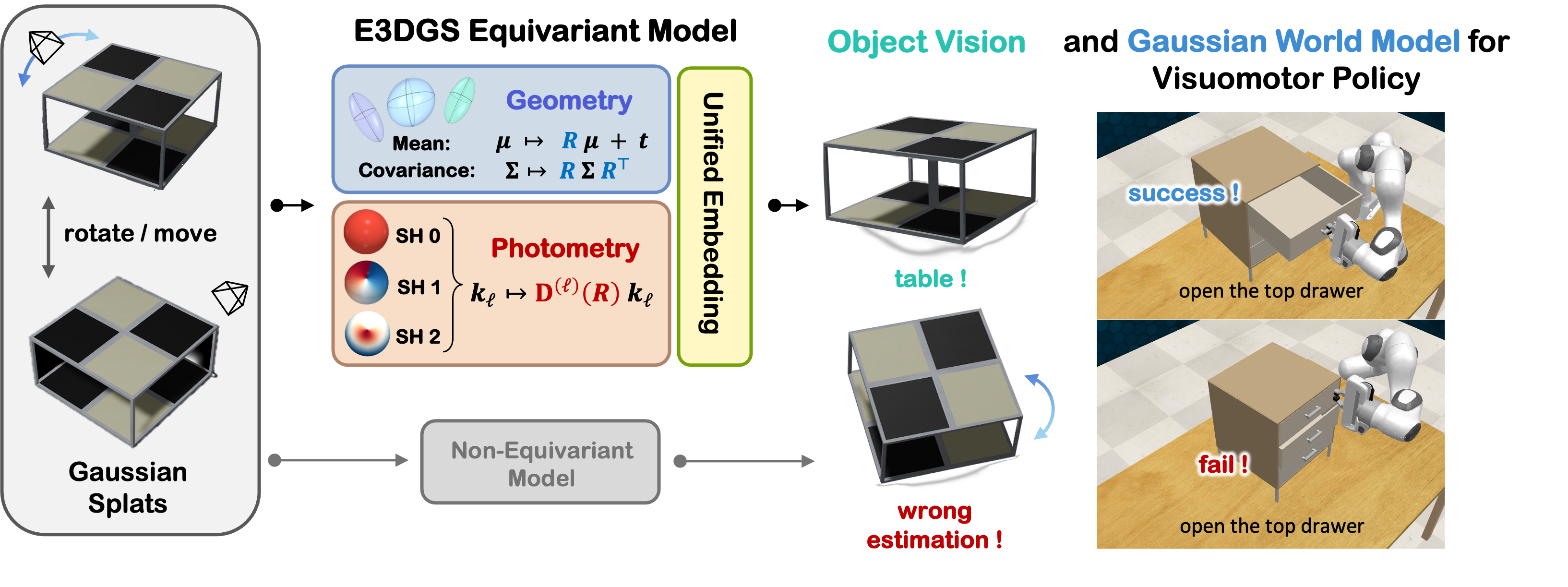}
\caption{\textbf{Overview of E3DGS.}
3D Gaussian primitives couple geometry (mean/covariance) and view-dependent photometry (SH), which transform under camera-frame changes by heterogeneous $\SO(3)$ rules.
E3DGS lifts both into a unified matrix carrier and processes them with a rigid-body equivariant backbone, enabling robust object recognition and action-conditioned Gaussian world modeling.}
\label{fig:fig1}
\end{figure*}

\section{Introduction}

The physical world is not merely a collection of shapes; it is an interplay of geometry and light. 3D Gaussian Splatting (3DGS)~\cite{kerbl2023gaussiansplatting} explicitly models this duality by assigning Spherical Harmonic (SH) coefficients to each primitive to capture view-dependent appearance. Yet, in the context of deep representation learning, this rich photometric information is often discarded. Current state-of-the-art methods (\eg, \cite{ma2024shapesplat, li2025scenesplat, lu2024manigaussian, ye2023gaussian}) strip away view-dependent SH coefficients in their 3DGS-based learning frameworks to avoid the complexities of rotational symmetries, reducing a glossy, textured scene to a matte, colored point cloud. This architectural compromise inherently limits the model's capacity to distinguish objects with identical geometries but distinct material properties (\eg, a matte plastic cup versus a shiny metallic mug). 

This simplification stems from a fundamental bottleneck in 3D deep learning: strict $\mathrm{SE}(3)$ equivariance. While translation equivariance can be elegantly achieved using relative coordinates, rotation poses a severe challenge. Geometric attributes like position ($\mathbf{x}$) and covariance ($\mathbf{\Sigma}$) admit standard tensor transformation rules under rotation $R \in \mathrm{SO}(3)$, whereas SH coefficients transform via \mbox{Wigner-$D$} matrices. Integrating these disparate data types into a single equivariant pipeline typically requires heavy machinery, such as Tensor Field Networks (TFNs)~\cite{thomas2018tfn}. For scenes parameterized by millions of Gaussian primitives, relying on Clebsch--Gordan tensor products to process these mixed representations becomes computationally prohibitive~\cite{fuchs2020se3transformer}. Consequently, existing 3DGS learning frameworks lack a unified equivariant formulation.

To achieve rotation robustness, prior works either rely on extensive data augmentation or bypass the symmetry problem entirely by discarding all view-dependent SH coefficients (degrees $\ell \ge 1$). By reducing appearance solely to the 0-th order SH, they treat photometry as a static scalar signal. This simplification fundamentally fails to model the view-dependent complexities of the real world. 

To resolve this representation mismatch, we propose a paradigm shift: \textbf{treating photometry as geometry}. For spherical harmonics of degrees $\ell \le 2$, the coefficient spaces $\mathbb{R}^{2\ell+1}$ realize the same irreducible representations ($V_0 \oplus V_1 \oplus V_2$) that appear in the $3 \times 3$ matrix space $\mathfrak{gl}(3)$ under conjugation. We construct linear intertwiners $\Phi_\ell$ that map SH coefficients directly into these corresponding matrix subspaces, replacing the Wigner-$D$ action with the simple conjugation rule that governs geometric tensors. By lifting each Gaussian primitive into a stack of $3 \times 3$ matrix channels, all rotation-active attributes—both geometric and photometric—transform via a single, unified operation: $M \mapsto RMR^\top$. \cref{fig:fig1} summarizes this unification and the resulting applications to object recognition and action-conditioned world modeling.

Building upon this \textit{Color-as-Geometry} formulation, we introduce \textbf{E3DGS}.
To the best of our knowledge, this is the \textbf{first formulation that makes the view-dependent appearance of 3DGS an equivariant geometric object}, allowing geometry and photometry to transform under a single conjugation law. Our core contributions are:
\begin{enumerate}[leftmargin=*]
    \item \textbf{Color-as-Geometry.} We show that view-dependent photometry is itself an equivariant geometric object: an explicit intertwiner identifies the SH$(0$--$2)$ coefficients with the conjugation carrier $\mathfrak{gl}(3)$, so that appearance and geometry of a 3D Gaussian transform by the single rule $M\mapsto RMR^\top$. This is, to our knowledge, the first time view-dependent appearance is treated as an equivariant tensor for learning on 3DGS.

    \item \textbf{A complete family of matrix carriers.} We prove that every spherical-harmonic degree can be represented exactly as matrix conjugation on a carrier $\End(V_k)$, and we identify which degrees each carrier covers. The $3\times3$ matrix used in our model is the smallest such carrier and covers view-dependent appearance up to degree $2$; higher degrees are handled by larger carriers in the same family.

    \item \textbf{Adjoint-equivariant transformer for 3DGS.} We instantiate a matrix-native rigid-body ($\SE(3)$) equivariant encoder with $\Ad$-equivariant operations (\textbf{ReLN-Attention}, \textbf{ReLN-LayerNorm}) whose pairwise interaction reduces to an $\SO(3)$-invariant scalar contraction with $O(1)$ per-pair cost, avoiding the Clebsch--Gordan tensor products whose cost grows steeply with representation degree.

    \item \textbf{Validation on 3DGS benchmarks.} We instantiate the encoder as an Equivariant Gaussian MAE and as an action-conditioned Gaussian world model, and evaluate on object-level recognition and action-conditioned manipulation, demonstrating rotation robustness and label efficiency across camera-frame changes.
\end{enumerate}

\section{Related Work}

\subsection{3D Gaussian Splatting and the Representation Mismatch}
3D Gaussian Splatting (3DGS) parameterizes scenes using anisotropic Gaussians with explicit geometry (means, covariances) and view-dependent photometry (Spherical Harmonics, SH) \cite{kerbl2023gaussiansplatting,mildenhall2020nerf}. Recent feed-forward models predict these primitives directly from sparse observations, establishing 3DGS as a reusable scene representation for downstream perception and decision-making \cite{charatan2023pixelsplat,szymanowicz2023splatterimage,tang2024lgm}. However, this explicit parameterization introduces a representational mismatch under different camera poses or spatial transformations: geometric attributes follow standard $\SE(3)$ or $\SO(3)$ tensor rules, while SH coefficients transform via Wigner-$D$ matrices \cite{varshalovich1988angularmomentum,ivanic1996rotation,pinchon2007rotation}. Consequently, direct fusion of these attributes either breaks rotational symmetry or necessitates computationally intensive tensor-product pipelines.

\subsection{Representation Learning and Downstream Tasks on 3DGS}
A growing body of work processes 3DGS primitives as direct inputs for scene understanding. Object-centric and scene-level pipelines, such as ShapeSplat and SceneSplat, enable representation learning and vision--language pretraining directly on Gaussians \cite{ma2024shapesplat,li2025scenesplat,ma2025scenesplatpp}. Concurrently, 3DGS supports open-vocabulary segmentation, semantic querying, and interactive editing by embedding language-aligned features or optimizing prompt-driven graph cuts \cite{shi2023languagegaussians,guo2024semanticgaussians,peng2024vlgs,li2025semanticsplat,siegel2025segsplat,cen2025saga,choi2024clickgaussian,jain2024gaussiancut,shen2024flashsplat,li2025refersplat}. Dynamic 3DGS has also been integrated into language-conditioned robotic manipulation frameworks, such as ManiGaussian \cite{lu2024manigaussian}. Despite these advances, existing methods typically circumvent the complexity of joint geometric--photometric rotation. They either reduce photometry to view-independent components (SH0), discard higher-order effects, or rely on heuristic data augmentation to approximate rotational robustness. Our work provides the missing symmetry constraint, enabling joint equivariant learning of geometry and higher-order photometry.

\subsection{Equivariant Architectures and Matrix-Based Alternatives}
Equivariance provides an inductive bias for robustness and data efficiency when task-relevant symmetries are present, and it has become a standard design principle in modern geometric deep learning~\cite{bronstein2021geometricdl,cohen2016group,wang2022extrinsic,wang2022eqinv}. In 3D deep learning, standard approaches achieve $\SE(3)$ equivariance using irreducible representations and Clebsch--Gordan (CG) tensor products \cite{thomas2018tfn,fuchs2020se3transformer,weiler20183dsteerable,cohen2018spherical,geiger2022e3nn}. While highly expressive, the cost of these tensor products scales steeply with the representation degree, creating runtime and memory bottlenecks \cite{tholke2022torchmdnet,batatia2022mace,xie2025the,liao2023equiformerv2}. To improve scalability, alternative frameworks explore vector-based models, message passing, and Lie-group convolutions \cite{deng2021vectorneurons,satorras2021egnn,finzi2020lieconv}.

\paragraph{Cartesian-tensor networks.}
A complementary and closely related line replaces spherical-harmonic/CG machinery with \emph{Cartesian} tensor features. Most directly related to us, TensorNet~\cite{simeon2023tensornet} represents atomic environments with rank-2 Cartesian tensors and decomposes them into scalar, vector, and symmetric trace-free parts---the same $V_0\oplus V_1\oplus V_2$ irreducible content carried by $\mathfrak{gl}(3)$ under conjugation---and processes them through matrix products. Our construction differs in \emph{what is placed in the carrier}: Cartesian-tensor potentials build tensors from geometric (relative-position) features, whereas we identify the view-dependent \emph{photometric} SH coefficients themselves with the carrier through an explicit intertwiner (``color-as-geometry''), unifying appearance and geometry of a 3D Gaussian in a single conjugation law $M\mapsto RMR^\top$. We further show (\cref{thm:general_carrier}) that any SH degree is realized by conjugation on a carrier $\End(V_k)$, so the $3\times3$ Cartesian-tensor case is the smallest member of a complete family.

\paragraph{Matrix- and Lie-algebraic equivariant models.}
Beyond Cartesian tensors, matrix-feature and adjoint-equivariant Lie models process features through matrix operations for improved efficiency, numerical stability, or generality: EMLP~\cite{finzi2021emlp} solves the equivariant-linear-layer constraints for arbitrary matrix groups, and Lie Neuron / Reductive Lie Neuron (ReLN) constructions~\cite{lin2023lieneurons,batatia2023reductivelie,kim2025relns} target exact $\mathrm{GL}(n)$-equivariance with structured operations on Lie-algebra-valued features. Our backbone reuses ReLN-style $\Ad$-equivariant primitives, but the ReLN framework itself~\cite{kim2025relns} sidesteps photometric rotation by reducing appearance to the view-independent (invariant) SH$0$ component. In contrast, our contribution is the representation-theoretic identification of higher-order view-dependent photometry, SH$(0\text{--}2)$, with the $3\times3$ conjugation carrier, and its instantiation for 3DGS; this is a non-trivial step beyond applying an equivariant backbone to scalar colors.

Underlying these constructions is the classical correspondence between low-degree spherical harmonics and low-order tensors---including the symmetric trace-free component that governs Wigner-$D$ rotations and practical real-SH rotation formulas \cite{thorne1980multipole,varshalovich1988angularmomentum,ivanic1996rotation,pinchon2007rotation}. We make this correspondence explicit and operational by constructing an intertwiner from SH$(0\text{--}2)$ into $\mathfrak{gl}(3)$, so that geometry and photometry are processed uniformly and equivariantly under rigid-body ($\SE(3)$) motions, without CG tensor products in the feature pipeline.

\section{Unified Geometric--Photometric Equivariance via a Reductive Lie-Algebra Carrier}
\label{sec:unified_equivariance}

A 3D Gaussian primitive couples geometry (mean, covariance) with view-dependent photometry (SH). Under a camera-frame change $g=(R,t)\in\SE(3)$, the mean transforms via translation and rotation ($\mu\mapsto R\mu+t$), the covariance transforms by congruence ($\Sigma\mapsto R\Sigma R^\top$), and SH coefficients transform degree-wise via Wigner-$D$ matrices ($f_\ell\mapsto D^\ell(R)f_\ell$). These heterogeneous $\SO(3)$ representations obstruct a single equivariant processing pipeline. In this section, we prove that for SH degrees $\ell\le2$, both geometry and photometry admit an exact, unified realization within the reductive Lie algebra $\mathfrak{gl}(3)$ under matrix conjugation.

\subsection{3D Gaussian Primitives and Their Group Actions}
\label{subsec:primitive_actions}

We parameterize a 3D Gaussian primitive as $\mathcal{G}_i = \bigl(\mu_i, \Sigma_i, \alpha_i, \{f^{(c)}_{i,\ell}\}_{\ell=0}^{2}\bigr)$, where $\mu_i\in\mathbb{R}^3$ is the mean, $\Sigma_i\in\SPD(3)$ is the covariance, $\alpha_i\in\mathbb{R}$ is the opacity, and $f^{(c)}_{i,\ell}\in\mathbb{R}^{2\ell+1}$ denotes the degree-$\ell$ real SH coefficients for color channel $c$.

For a spatial rotation $R\in\SO(3)$, geometric attributes transform as $\mu_i \mapsto R\mu_i$, $\Sigma_i \mapsto R\Sigma_iR^\top$, and $\alpha_i \mapsto \alpha_i$, whereas SH coefficients transform as $f^{(c)}_{i,\ell} \mapsto D^\ell(R)\,f^{(c)}_{i,\ell}$. To isolate the rotational component under $\SE(3)$ motions, we operate on centered coordinates $\bar{\mu}_i := \mu_i - \frac{1}{N}\sum_{j=1}^N \mu_j$, ensuring $\bar{\mu}_i\mapsto R\bar{\mu}_i$ \cite{fuchs2020se3transformer, geiger2022e3nn}.

\subsection{A Unified $\mathfrak{gl}(3)$ Carrier for Geometry and Photometry}
\label{subsec:gl3_carrier}
To unify these heterogeneous actions, we adopt the matrix space $\mathfrak{gl}(3)=\mathbb{R}^{3\times 3}$ equipped with the standard conjugation action $\Ad_R(M) := RMR^\top$ for $R\in\SO(3)$. Under this adjoint action, $\mathfrak{gl}(3)$ decomposes into $\SO(3)$-irreducible subrepresentations:
\begin{equation}
\mathfrak{gl}(3)=\langle I\rangle \oplus \mathfrak{so}(3)\oplus \Sym_0(3),
\label{eq:gl3_irrep_decomp}
\end{equation}
where $\langle I\rangle=\{sI:s\in\mathbb{R}\}$ represents scalar matrices, $\mathfrak{so}(3)=\{A:A^\top=-A\}$ represents skew-symmetric matrices, and $\Sym_0(3)=\{S:S^\top=S,\ \tr(S)=0\}$ represents symmetric trace-free matrices. This decomposition matches the $\SO(3)$ irreducible types carried by spherical harmonics up to degree $\ell=2$. 

\paragraph{Remark (Exact capacity of the $3 \times 3$ carrier).}
Since $\mathfrak{gl}(3)\cong \End(\R^3)\cong V_1\otimes V_1\cong V_0\oplus V_1\oplus V_2$ as $\SO(3)$-representations under conjugation,
no $V_3$ component occurs, so the $3\times 3$ carrier realizes exactly the SH types $\ell\le 2$ and no more. As we show next (\cref{thm:general_carrier}), the same conjugation principle extends to every SH degree: each finite degree is realized as conjugation on a suitable carrier $\End(V_k)$, and the $3\times3$ carrier is simply the smallest one.



\begin{theorem}[Equivariant matrix realization of SH degrees $\ell\le2$]
\label{thm:sh_to_gl3_intertwiner}
Fix an orthonormal spherical-harmonic basis, and let $D^\ell(R)$ denote the corresponding real Wigner-$D$ matrices.
For each $\ell\in\{0,1,2\}$, there exists a linear $\SO(3)$-equivariant map
$\Phi_\ell:\R^{2\ell+1}\rightarrow V_\ell\subset \mathfrak{gl}(3)$ such that for all $R\in\SO(3)$ and $f_\ell\in\R^{2\ell+1}$,
\begin{equation}
\Phi_\ell\!\bigl(D^\ell(R)f_\ell\bigr) = \Ad_R\!\bigl(\Phi_\ell(f_\ell)\bigr)=R\,\Phi_\ell(f_\ell)\,R^\top.
\label{eq:phi_intertwining}
\end{equation}
Consequently, $\Phi_{\le2}:=\Phi_0\oplus\Phi_1\oplus\Phi_2$ identifies SH$(0$--$2)$ with a single $\mathfrak{gl}(3)$-valued representation
transforming by conjugation.
\end{theorem}
One concrete construction including basis conventions is given in the supplementary material.

\paragraph{General carrier: every SH degree is a matrix under conjugation.}
The construction above is the minimal instance of a general principle. For any degree-$k$ carrier $V_k$ with Wigner matrices $D^k(R)$, the conjugation representation $X\mapsto D^k(R)\,X\,D^k(R)^{-1}$ on $\End(V_k)$ decomposes as $\End(V_k)\cong V_k\otimes V_k\cong\bigoplus_{\ell=0}^{2k}V_\ell$.

\begin{theorem}[General matrix realization of spherical harmonics]
\label{thm:general_carrier}
For every SH degree $L$ there exist a carrier $V_k$ with $k=\lceil L/2\rceil$ and a linear, $\SO(3)$-equivariant injection $\iota_L:V_L\hookrightarrow \End(V_k)$, unique up to scale, such that for all $R\in\SO(3)$ and $f_L\in V_L$,
\begin{equation}
\iota_L\!\bigl(D^L(R)\,f_L\bigr)=D^k(R)\,\iota_L(f_L)\,D^k(R)^{-1}.
\end{equation}
Thus the Wigner-$D^L$ action on any spherical-harmonic band is realized \emph{exactly} as matrix conjugation, without Clebsch--Gordan tensor products at runtime. The model used in this paper is the case $k=1$, the unique geometry-native carrier ($D^1(R)=R$), which realizes SH$(0$--$2)$; the next band SH$3$ is realized in $\End(V_2)$ ($5\times5$). A constructive proof and a numerical check are given in \cref{app:general_carrier}.
\end{theorem}

Thus the $3\times3$ carrier is the smallest member of a complete family of matrix carriers: it covers SH degrees $\ell\le 2$, and larger carriers $\End(V_k)$ cover higher degrees.

\subsection{Unified Matrix Embedding of Gaussian Attributes into $\mathfrak{gl}(3)$}
\label{subsec:unified_lifting}
Leveraging Theorem~\ref{thm:sh_to_gl3_intertwiner}, we define a lifting mapping that embeds Gaussian attributes into corresponding $\mathfrak{gl}(3)$ subspaces:


\subsubsection{Covariance lift.}
Let $\Sym(3):=\{S\in\mathbb{R}^{3\times 3}:S^\top=S\}$ and $\SPD(3):=\{S\in\Sym(3):x^\top S x>0\ \ \forall x\neq 0\}$ be the space of symmetric positive-definite matrices. We embed the covariance via the matrix logarithm:
\begin{equation}
C_i := \log(\Sigma_i) \in \Sym(3) \subset \mathfrak{gl}(3).
\label{eq:cov_direct_lift}
\end{equation}
Under a spatial rotation $R\in\SO(3)$, the covariance transforms by congruence ($\Sigma_i \mapsto R\Sigma_iR^\top$). As the matrix logarithm commutes with orthogonal transformations, the lifted covariance preserves the adjoint equivariance:
\begin{equation}
    \log(R\Sigma_iR^\top) = R\log(\Sigma_i)R^\top = \Ad_R(C_i).
\label{eq:cov_congruence}
\end{equation}
This logarithmic lifting maintains exact $\SO(3)$ symmetry and maps the multiplicative $\SPD(3)$ manifold into a flat vector space $\Sym(3)$, improving numerical stability during network forward passes \cite{kim2025relns}.

\subsubsection{Position lift.}
Using the centered position $\bar\mu_i$ (Section~\ref{subsec:primitive_actions}), we lift the translational component via the hat map:
\begin{equation}
P_i := \widehat{\bar{\mu}_i} \in \mathfrak{so}(3), \quad \text{which natively satisfies} \quad \widehat{R\bar{\mu}_i} = R\widehat{\bar{\mu}_i}R^\top.
\label{eq:centered_position_hat}
\end{equation}

\subsubsection{Photometry lift.}
For each color channel $c\in\{r,g,b\}$, we lift SH coefficients into a single matrix
\begin{equation}
S_i^{(c)} := \Phi_{\le2}\!\left(f_{i,\le2}^{(c)}\right)\in \mathfrak{gl}(3),
\label{eq:sh_lift_per_color}
\end{equation}
where $f_{i,\le2}^{(c)}=(f_{i,0}^{(c)},f_{i,1}^{(c)},f_{i,2}^{(c)})$ stacks degrees $\ell\le2$.

\subsubsection{Equivariant Features and Type-0 Invariant Scalar Features.}
We collect the equivariant features as $H_i := [ P_i,\; C_i,\; S_i^{(r)},\; S_i^{(g)},\; S_i^{(b)} ] \in \mathfrak{gl}(3)^{C_{\mathrm{eq}}}$. Alongside these tensor channels, we maintain a separate Type-0 branch for $\SO(3)$-invariant scalars, defined as $s_i := [ \alpha_i,\; f_{i,0}^{(r)},\; f_{i,0}^{(g)},\; f_{i,0}^{(b)},\; E_{\mathrm{task}} ]^\top \in \mathbb{R}^{C_0}$, where $\alpha_i$ is the opacity, $f_{i,0}^{(c)}$ denotes the diffuse color, and $E_{\mathrm{task}}$ represents task-specific semantics (\eg, language instruction embeddings). During the network's forward pass, invariant scalars derived from $H_i$ (\eg, $\tilde{B}$-contractions) are used to gate matrix channels without breaking equivariance.

\begin{proposition}[Equivariance of the lifted representation]
\label{prop:unified_conjugation}
For $R\in\SO(3)$, each matrix channel $(H_i)_k\in\mathfrak{gl}(3)$ transforms as
$(H_i)_k\mapsto \Ad_R((H_i)_k)=R(H_i)_kR^\top$, while $s_i$ is invariant.
\end{proposition}

\begin{proof}
The claim follows from \eqref{eq:centered_position_hat}, \eqref{eq:cov_congruence}, and the intertwining property \eqref{eq:phi_intertwining}.
\end{proof}


\subsection{ReLN Processing on the Unified Carrier}
\label{subsec:reln_on_gl3}
Given the lifted representations $(H, s)$ of the 3D Gaussian primitives, we construct an encoder $\mathcal{F}_\theta$ utilizing $\Ad$-equivariant Reductive Lie Neurons (ReLN) building blocks \cite{kim2025relns}. To parameterize nonlinearities without breaking symmetry, we construct $\Ad$-invariant scalars using the non-degenerate modified Killing form on $\mathfrak{gl}(3)$: $\tilde{B}(X,Y)=6\,\tr(XY)-\tr(X)\tr(Y)$. By stacking equivariant linear and nonlinear layers, the entire block commutes with $\Ad$.

\begin{theorem}[Equivariance of the Neural Encoder]
\label{thm:reln_on_gl3_equiv}
Let $H$ and $s$ denote the equivariant features and invariant scalar features of the input 3D Gaussians. An encoder $\mathcal{F}_\theta(H, s) = (H', s')$ composed of $\Ad$-equivariant ReLN primitives and invariant-weighted aggregations satisfies:
\begin{equation}
    \mathcal{F}_\theta(\Ad_R(H),\, s) = \big(\Ad_R(H'),\, s'\big) \quad \forall R\in\SO(3),
\end{equation}
where $\Ad_R$ denotes channel-wise conjugation.
\end{theorem}

As the scalar features $s$ are $\SO(3)$-invariant, scaling or gating the equivariant channels by these scalars commutes with the adjoint action: $s \cdot \Ad_R(H) = \Ad_R(s \cdot H)$. This design guarantees exact global equivariance while leveraging the scalar context.

\section{Adjoint-Equivariant Architecture for 3D Gaussian Splatting}
\label{sec:reln_3dgs_arch}

\begin{figure*}[t]
\centering
\includegraphics[width=\textwidth]{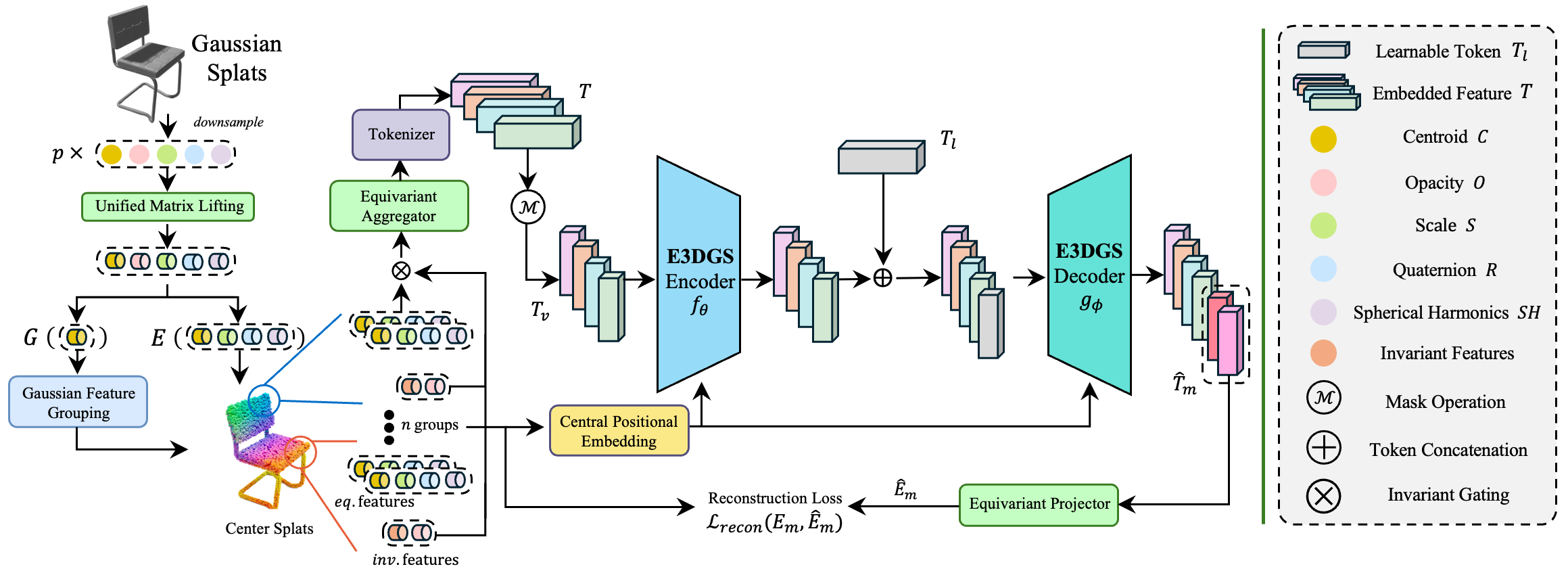}
\caption{\textbf{Architecture of the Base E3DGS Gaussian Masked Autoencoder.} We instantiate our framework as a Gaussian MAE for object vision tasks (\cref{subsec:shapesplat_results}). Input 3D Gaussians undergo Unified Matrix Lifting to embed geometry and photometry into a single carrier. Following Gaussian Feature Grouping, scalar invariants modulate the equivariant features via \textit{invariant gating} ($\otimes$). The tokens are partially masked, processed by an equivariant autoencoder ($f_\theta, g_\phi$), and reconstructed via an Equivariant Projector to compute $\mathcal{L}_{\mathrm{recon}}$.}
\label{fig:main_architecture}
\end{figure*}

Building on the unified $\mathfrak{gl}(3)$ carrier (\cref{sec:unified_equivariance} and \cref{fig:fig1}), we introduce \textbf{E3DGS}, an end-to-end equivariant architecture for representation learning on 3DGS. We instantiate E3DGS in two regimes: a Gaussian Masked Autoencoder for large-scale self-supervised pretraining and object recognition (\cref{fig:main_architecture}), and an action-conditioned Gaussian world model for robotic manipulation.

E3DGS consists of three stages: (i) a unified lifting that embeds each primitive into a single $\mathfrak{gl}(3)$ matrix carrier, while collecting rotation-invariant scalars for gating; (ii) an $\Ad$-equivariant ReLN backbone $\mathcal{F}_\theta$ that mixes matrix channels and applies invariant-gated updates; and (iii) task-specific heads with inverse lifting for geometric outputs.


\subsection{Backbone Instantiation via $\Ad$-Equivariant Operations}
\label{subsec:arch_operator_replacement}
We construct the E3DGS backbone by enforcing $\Ad$-equivariance on the $\mathfrak{gl}(3)$-valued features, systematically replacing standard Transformer operations that violate rotational symmetry with Lie-algebraic counterparts.

\subsubsection{ReLN-Attention.}
In standard Transformers, attention scores rely on inner products that break under arbitrary rotations if applied directly to geometric coordinates. To resolve this, we propose \textit{ReLN-Attention}, an equivariant self-attention mechanism native to the $\mathfrak{gl}(3)$ manifold. We project the tokenized inputs into $\Ad$-equivariant queries, keys, and values ($Q, K, V \in \mathfrak{gl}(3)$). We construct the attention weights $a_{ij}$ using the modified Killing form $\tilde{B}$ as the invariant similarity measure between query and key matrices:
\begin{equation}
    a_{ij} \;=\; \mathrm{softmax}_j \Big( \tilde{B}(Q_i, K_j) \Big).
\label{eq:arch_attn_update}
\end{equation}
Because the modified Killing form is invariant under conjugation ($\tilde{B}(\Ad_R Q, \Ad_R K) = \tilde{B}(Q, K)$), the computed scalar weights $a_{ij}$ are invariant. Consequently, the linear combination $\sum a_{ij} V_j$ preserves the $\Ad$-equivariance of the value matrices during token mixing.

\subsubsection{Equivariant Layer Normalization (ReLN-LayerNorm).}
To stabilize training without destroying structural symmetry, we generalize VN-LayerNorm~\cite{assaad2023vntransformer} from 3D vectors to the $\mathfrak{gl}(3)$ manifold, introducing \textbf{ReLN-LayerNorm}. We compute the scalar magnitude of each $3 \times 3$ channel using the Frobenius norm—a positive-definite invariant metric on $\mathfrak{gl}(3)$. We apply standard LayerNorm exclusively to these magnitudes and rescale the original matrices:
\begin{equation}
    X'_c \;=\; \frac{\mathrm{LN}\big(\|X_c\|_F\big)}{\|X_c\|_F + \epsilon} X_c,
\end{equation}
where $\mathrm{LN}$ operates over the channel dimension $c$. As $\|\Ad_R(X_c)\|_F = \|X_c\|_F$, the scalar multiplication stabilizes variance while preserving geometric orientation.

\subsubsection{Equivariant Feed-Forward (ReLN-Linear, ReLN-ReLU).}
We replace standard feed-forward networks with \textit{ReLN-Linear} and \textit{ReLN-ReLU} operations~\cite{kim2025relns} to induce pointwise nonlinearities while maintaining equivariance under the $\Ad$ action. Instead of element-wise activations, the network processes matrix channels through $\Ad$-equivariant linear mixings (ReLN-Linear) and invariant-gated activations based on the modified Killing form $\tilde{B}$ (ReLN-ReLU). Stacking equivariant linear and nonlinear layers preserves exact rotational symmetry.

\subsection{Task Heads and Inverse Geometric Lifting}
\label{subsec:task_heads_and_lifting}
The downstream task dictates the symmetry requirements of the prediction head.

\paragraph{\textbf{Invariant Readout.}}
For invariant tasks (\eg, classification, global property prediction), we extract scalar invariants from the matrix channels via $\tilde{B}$-based self-contractions~\cite{kim2025relns}. These are concatenated with the scalar branch $s_i$, pooled, and processed by a standard MLP to yield rotation-invariant predictions.

\paragraph{\textbf{Equivariant Prediction \& Inverse Lifting.}}
For geometric predictions (\eg, dynamics), the network outputs conjugated matrices $H_{\mathrm{out}}\in\mathfrak{gl}(3)$. To map latent spatial channels back to physical $\mathbb{R}^3$ vectors, we project them onto $\mathfrak{so}(3)$ and apply the inverse hat map ($\vee$):
\begin{equation}
    \tilde{v}_i = \Big(\mathrm{skew}\big((H_{\mathrm{out}})_{\mathrm{pos}}\big)\Big)^\vee \in \mathbb{R}^3,
\label{eq:inv_vec_readout}
\end{equation}
where $\mathrm{skew}(A)=\tfrac12(A-A^\top)$. Because both the skew projection and the vee map commute with orthogonal transformations, this operation guarantees exact $\SO(3)$-equivariance ($\tilde{v}_i \mapsto R\,\tilde{v}_i$).


\section{Experiments}
\label{sec:experiments}

We evaluate the \textbf{E3DGS} in two regimes that apply 3D Gaussian primitives: object-level recognition (\cref{subsec:shapesplat_results}) and action-conditioned Gaussian world modeling (\cref{subsec:manigaussian_results}) for visuomotor control. We adhere to the standard evaluation protocols of the respective benchmarks to ensure controlled comparisons.

\paragraph{Network Implementations and Scalability.}
We instantiate the E3DGS backbone by lifting spatial and photometric attributes into unified $3\times3$ matrix channels. To align capacity with the underlying 3D carrier, we use a hidden width of $C/3$. This design reduces trainable parameters in practice by about $2.4$--$2.6\times$ across the reported settings (Appendix Table~\ref{tab:comp_cost}), because matrix-valued channels are mixed along the channel axis rather than treated as nine independent scalar channels. At the same time, the matrix contractions and invariant bilinear operations increase GMACs, so we describe the method as parameter-efficient rather than broadly runtime-efficient.

\paragraph{Object Vision: Gaussian-MAE.}
To assess object recognition from 3D Gaussian splats, we evaluate our E3DGS-MAE architecture using the ShapeSplat dataset and evaluation protocol~\cite{ma2024shapesplat}. By training our equivariant transformer backbone under this established pretraining--fine-tuning setup, we isolate the contributions of available SH degrees ($\ell\le 0,1,2$) and measure (i) robustness to coordinate frame shifts and (ii) label efficiency in low-data regimes.

\paragraph{Gaussian World Model: RLBench.}
To determine if the equivariant representation improves action-conditioned dynamics modeling, we design an E3DGS-based world model and evaluate it within the ManiGaussian~\cite{lu2024manigaussian} framework. We utilize the RLBench multi-task setup (10 tasks, 166 variations) to measure the impact of geometric regularization on manipulation task success rates.


\subsection{Object recognition on ShapeSplat (Gaussian-MAE)}
\label{subsec:shapesplat_results}

We evaluate object classification on ModelNet10-4K using the Gaussian-MAE pretraining and finetuning pipeline~\cite{ma2024shapesplat}. We instantiate our E3DGS-MAE in the \textsc{Full} finetuning setting to update all network parameters.

\begin{table}[t]
\centering
\caption{Classification Accuracy on ModelNet10 (overall $\uparrow$[\%]). ID and SE(3) denote Identity and Special Euclidean group. Baselines are reproduced under identical configurations for fair comparison.}
\label{tab:rot_robust}
\begin{tabular*}{\linewidth}{@{\extracolsep{\fill}}lccc@{}}
\toprule
Method & ID / ID & ID / SE(3) & SE(3) / SE(3) \\ \midrule
Point-BERT \cite{yu2022pointbert} & 94.82 & \texttimes & \texttimes \\
Point-MAE \cite{pang2022pointmae} & 94.93 & \texttimes & \texttimes \\
Gaussian-MAE \cite{ma2024shapesplat} & 94.71 & 18.28 & 88.32 \\ 
Gaussian-MAE; SH (0,1)   & 94.27 & 16.27 & 89.87 \\ 
Gaussian-MAE; SH (0,1,2) & 94.05 & 12.89 &  88.87 \\ \midrule
E3DGS-MAE; SH (0) & 95.15 & 95.15 & 94.82 \\ 
E3DGS-MAE; SH (0,1) & \textbf{95.26} & \textbf{95.26} & \textbf{95.26} \\ 
E3DGS-MAE; SH (0,1,2) & 94.79 & 94.79 & 94.49 \\ \bottomrule
\end{tabular*}
\end{table}

For transformed evaluations, we use the active scene-rotation convention: rotating the Gaussian scene by $R$ sends directions $d\mapsto Rd$ and SH coefficients $f_{i,\ell}^{(c)}\mapsto D^\ell(R)f_{i,\ell}^{(c)}$, consistently with the mean and covariance transforms. A passive camera-coordinate change would instead use $D^\ell(R^\top)$.

\paragraph{Zero-Shot Robustness and SH Ablation.}
Table~\ref{tab:rot_robust} presents classification results under three pose distributions: \textsc{ID/ID} (canonical), \textsc{ID/$\SE(3)$} (zero-shot transformation), and \textsc{$\SE(3)$/$\SE(3)$} (pose-augmented). For fair comparison, baselines are reproduced using identical training recipes, epochs, and architectural capacities.

Under zero-shot transformations (\textsc{ID/$\SE(3)$}), baselines experience failure (e.g., G-MAE dropping to 18.28\%). In contrast, E3DGS-MAE maintains consistent accuracy, demonstrating equivariance. In the canonical \textsc{ID/ID} setting, E3DGS-MAE is competitive with prior point-based and reproduced Gaussian-based architectures, while its main advantage appears under zero-shot transformed evaluation.

Furthermore, concatenating higher-order photometry to the baseline drastically degrades zero-shot performance ($18.28\% \rightarrow 12.89\%$) due to overfitting to pose-correlated appearance. Conversely, E3DGS-MAE successfully assimilates view-dependent photometry to boost accuracy, confirming that spherical harmonics can be efficiently learned under group actions.

\begin{figure}[h]
\centering
\includegraphics[width=0.9\linewidth]{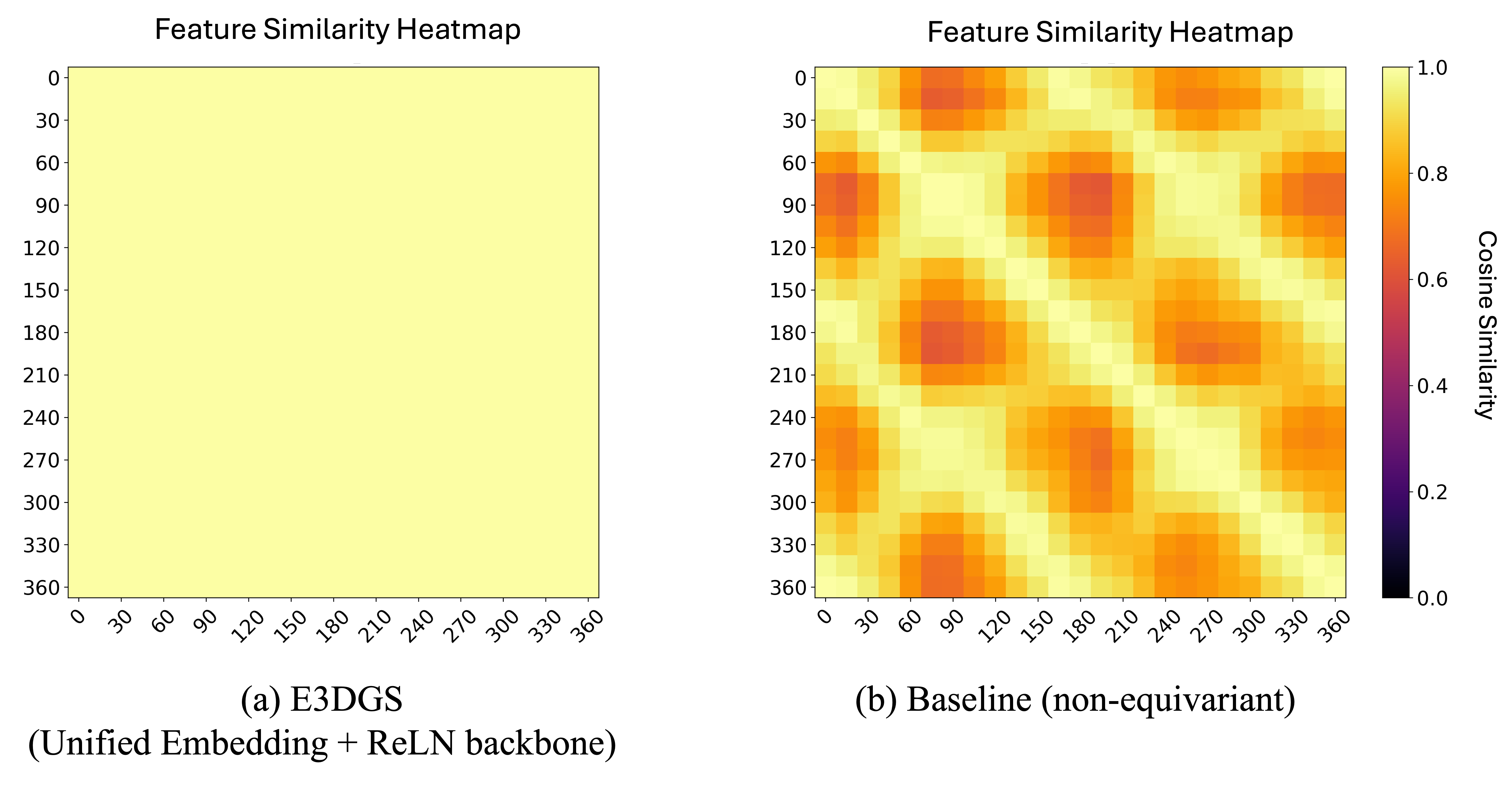}
\caption{\textbf{Pooled invariant descriptor similarity across input rotations.}
Pairwise cosine similarities of pooled invariant descriptors obtained after $\tilde{B}$-based contractions/readout for a test instance rotated $0^\circ$--$360^\circ$. E3DGS (left) remains stable after invariant readout, whereas G-MAE (right) exhibits severe representational drift.}
\label{fig:inv_heatmap}
\end{figure}

\paragraph{Qualitative Analysis of Reconstruction.}
To visually validate this robustness, Figure~\ref{fig:qualitative_mae} compares MAE reconstructions from the test set. While G-MAE succeeds in canonical poses, it collapses under arbitrary spatial rotations, exhibiting failure modes such as color corruption, structural geometry distortion, and missing topology. E3DGS-MAE, however, demonstrates consistent geometric and photometric reconstructions irrespective of the input orientation. This visual stability is quantitatively supported by consistently lower Chamfer Distance (CD) scores (inset in black, \cref{fig:qualitative_mae}), which measure the reconstruction error against the ground truth.

\paragraph{Analysis of Invariant Descriptor Stability.}
We track the globally pooled invariant descriptors obtained from the encoder features after $\tilde{B}$-based contractions, while continuously rotating a test instance from $0^\circ$ to $360^\circ$ (Fig.~\ref{fig:inv_heatmap}). Because the raw matrix-valued channels are equivariant and therefore rotate with the input, the cosine comparison is performed after invariant descriptor extraction. E3DGS-MAE yields a uniform similarity matrix at this invariant readout level, whereas the non-equivariant baseline exhibits angle-dependent representational drift.

\vspace{6pt}
\noindent
\begin{minipage}[t]{0.56\textwidth}
\paragraph{Data Efficiency via Equivariant Inductive Bias.}
Equivariance acts as a structural inductive bias, which typically reduces sample complexity. We evaluate data efficiency by finetuning on restricted subsets ($\{10\%, 30\%, 50\%, 100\%\}$) of the ModelNet10-4K training set. As shown in Figure~\ref{fig:data_eff}, E3DGS consistently outperforms the baseline across all subset ratios. The performance margin is prominent in the low-data regime, indicating that symmetry-consistent coupling prevents the network from overfitting to spurious pose correlations when labels are scarce.
\end{minipage}%
\hfill
\begin{minipage}[t]{0.36\textwidth}
    \vspace{0pt} 
    \centering
    \includegraphics[width=\linewidth]{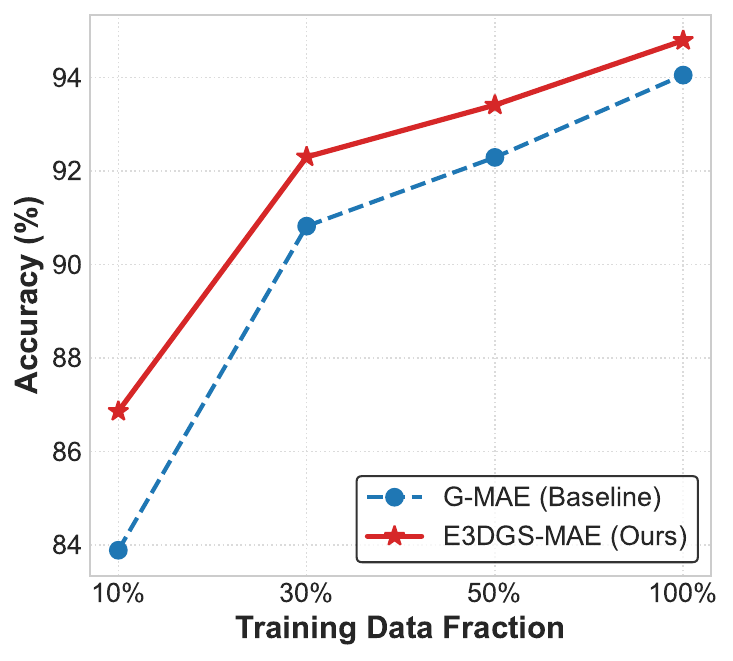}
    \captionof{figure}{\textbf{Data-efficiency.} E3DGS-MAE outperforms the baseline across restricted training subsets.}
    \label{fig:data_eff}
\end{minipage}

\begin{figure}[t]
\centering
\includegraphics[width=\linewidth]{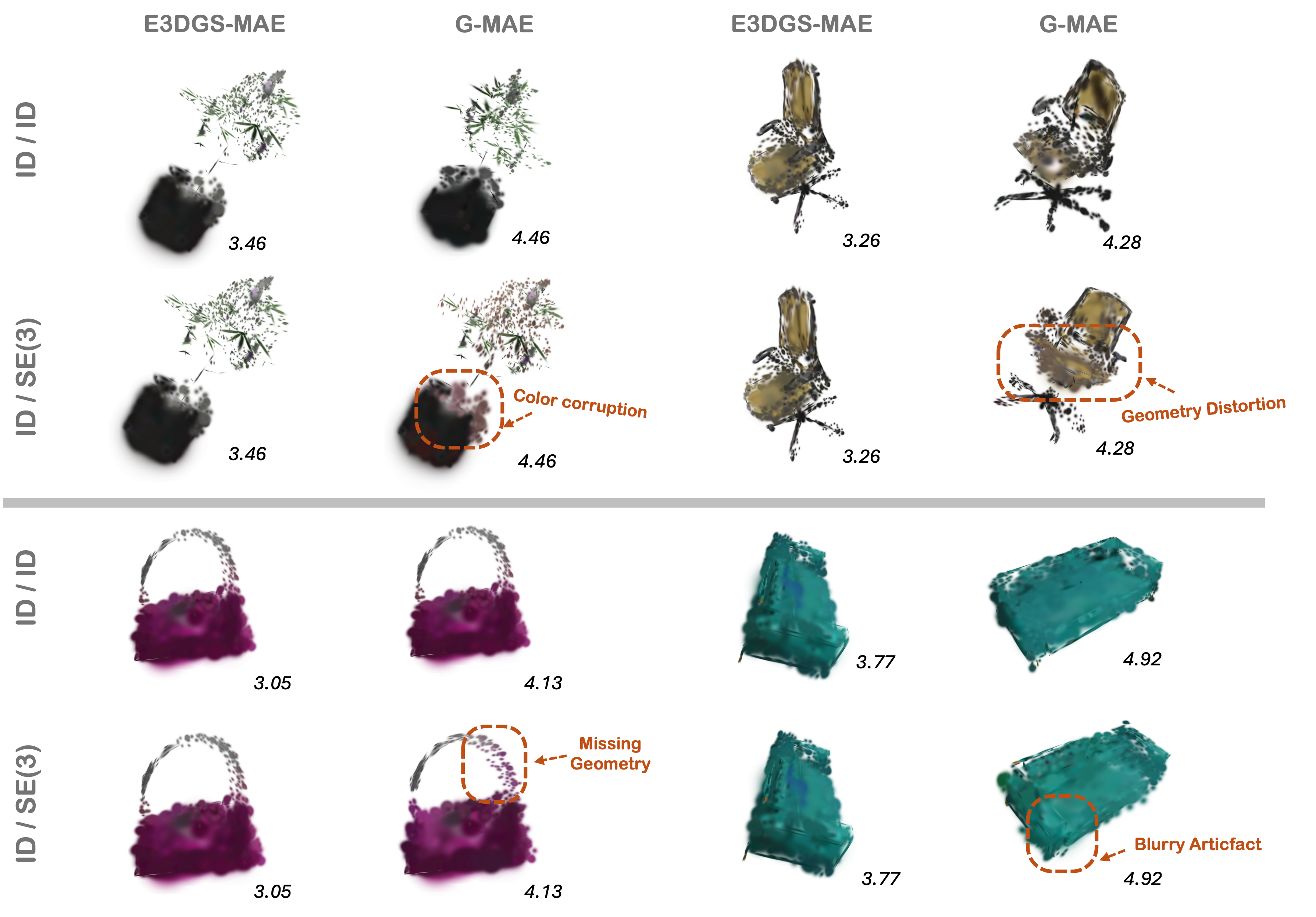} 
\caption{\textbf{Qualitative comparison under out-of-distribution transformations.} Under test-time rotations (\textsc{ID/$\SE(3)$}), the non-equivariant baseline (G-MAE) suffers severe representation collapse, exhibiting color corruption, missing topology, and geometric distortion (orange dashed boxes). Conversely, our equivariant E3DGS-MAE preserves structural fidelity and view-dependent appearance regardless of input orientation. Inset numbers denote Chamfer Distance ($\downarrow$).}
\label{fig:qualitative_mae}
\end{figure}


\subsection{Object Recognition on ModelNet40}
\label{subsec:modelnet40_eval}

We include ModelNet40 to test whether the proposed unified representation remains effective when the category space becomes larger and the object geometry is more diverse. All ModelNet40 experiments use the same pretraining and finetuning protocol as our ModelNet10 experiments (\cref{subsec:shapesplat_results}).

Table~\ref{tab:modelnet40_robust} reports accuracy under canonical and transformed evaluation. The main observation is that canonical accuracy alone does not determine pose robustness: a non-equivariant model can remain competitive on \textsc{ID/ID} while degrading sharply under transformed evaluation, indicating reliance on pose-specific information. By contrast, E3DGS maintains stable performance across all evaluated splits, consistent with the symmetry built into the active feature stream and the invariant classifier.

\paragraph{Interpretation of SH utilization.}
Introducing higher-order SH coefficients, as in SH$(0,1)$ and SH$(0,1,2)$, does not degrade the proposed equivariant representation on ModelNet40. In fact, the SH$(0,1)$ variant achieves the best performance among the E3DGS configurations, reaching $93.17\%$ on all three evaluation splits. The SH$(0,1,2)$ variant also improves over SH$(0)$ in the canonical setting ($92.57\%$ vs. $92.25\%$) while maintaining strong pose consistency.

\paragraph{Implication.}
The main takeaway from ModelNet40 is not only the improved canonical accuracy, but also the ability to maintain performance across transformed splits without relying on pose-specific cues. This is precisely the regime in which naive feature-space concatenation becomes restrictive: higher-order SH coefficients carry transformation-sensitive information, and handling them without an explicit equivariant structure leads to poor robustness once the test-time pose varies. The current ModelNet40 results therefore strongly support the claim that the benefit of the unified carrier lies in the pose-consistent learning of Gaussian geometry and appearance.

\begin{table}[h]
\centering
\caption{\textbf{Classification Accuracy on ModelNet40} (overall $\uparrow$[\%]). \textsc{ID} and \textsc{$\SE(3)$} denote the Identity and Special Euclidean group. Values for E3DGS-MAE and reproduced Gaussian-MAE baselines under spatial transformations are provided to highlight the pose consistency induced by our unified $\mathfrak{gl}(3)$ carrier.}
\label{tab:modelnet40_robust}
\begin{tabular*}{\linewidth}{@{\extracolsep{\fill}}lccc@{}}
\toprule
Method & ID / ID & ID / SE(3) & SE(3) / SE(3) \\ \midrule
\multicolumn{4}{@{}l}{\textit{Supervised Learning Only}} \\ \midrule
PointNet \cite{qi2017pointnet} & 89.20 & \texttimes & \texttimes \\
PointNet++ \cite{qi2017pointnet++} & 91.90 & \texttimes & \texttimes \\
PTv1 \cite{zhao2021point} & 90.60 & \texttimes & \texttimes \\
PTv2 \cite{wu2022point} & 91.60 & \texttimes & \texttimes \\ \midrule
\multicolumn{4}{@{}l}{\textit{with Self-Supervised Pretraining (FULL)}} \\ \midrule
Point-BERT \cite{yu2022pointbert} & 93.20 & \texttimes & \texttimes \\
Point-MAE \cite{pang2022pointmae} & 93.20 & \texttimes & \texttimes \\ \midrule
Gaussian-MAE \cite{ma2024shapesplat}; SH (0) & 92.21 & 6.81 & 90.09 \\ 
Gaussian-MAE \cite{ma2024shapesplat}; SH (0,1) & 92.50 & 3.28 & 91.12 \\ 
Gaussian-MAE \cite{ma2024shapesplat}; SH (0,1,2) & 92.31 & 4.90 & 90.39 \\ \midrule
{E3DGS-MAE (Ours)}; SH (0) & 92.25 & 92.25 & 92.02 \\ 
E3DGS-MAE (Ours); SH (0,1) & 93.17 & 93.17 & 93.17 \\ 
E3DGS-MAE (Ours); SH (0,1,2) & 92.57 & 92.57 & 92.25 \\ \bottomrule
\end{tabular*}
\end{table}

\subsection{Action-Conditioned Gaussian World Modeling on ManiGaussian}
\label{subsec:manigaussian_results}

While \cref{subsec:shapesplat_results} demonstrates the efficacy of our $\mathfrak{gl}(3)$ carrier for object perception, sequential action in robotic manipulation introduces the challenge of modeling environmental dynamics. Here, we evaluate whether our unified geometric--photometric representation improves \emph{predictive world modeling}. We instantiate our equivariant Gaussian world model within the ManiGaussian~\cite{lu2024manigaussian} framework and evaluate it on the RLBench multi-task benchmark.

\paragraph{Equivariant Action-Conditioned World Model in the Lifted Space.}
Given that the Gaussian world state $\mathcal{G}_t$ is lifted into equivariant matrix channels $H_t^{\mathrm{state}}$, we embed the manipulation action into the same algebraic space to preserve symmetry. Manipulation action $a_t = (v_t, q_t, c_t)$ comprises a 3D translation $v_t \in \mathbb{R}^3$, a rotation quaternion $q_t$, and an invariant gripper openness $c_t \in \mathbb{R}$. We convert $q_t$ to an axis-angle vector $\omega_t \in \mathbb{R}^3$. Actions are expressed in the scene/world frame; under a global rotation of the scene frame, both the translational command $v_t$ and the axis-angle vector $\omega_t$ transform as $v_t \mapsto R v_t$ and $\omega_t \mapsto R \omega_t$. Because both $v_t$ and $\omega_t$ transform as standard vectors ($x \mapsto Rx$), we use the hat map to lift them into $\mathfrak{so}(3)$ matrices: $H_t^{\mathrm{act}} = [ \widehat{v}_t, \, \widehat{\omega}_t ]$. Since $\widehat{Rx} = R \widehat{x} R^\top$, these lifted channels natively transform via adjoint conjugation.
We concatenate the state and action matrices to form a joint equivariant input. The world model $f_{\mathrm{wm}}^{\mathrm{eq}}$ processes the concatenated state and action matrices, modulated by the invariant gripper state $c_t$, to predict the state deformation matrix $\Delta H_{t+1}$:
\begin{equation}
    \Delta H_{t+1} = f_{\mathrm{wm}}^{\mathrm{eq}}\big( [H_t^{\mathrm{state}}, \, H_t^{\mathrm{act}}], \, c_t \big).
\end{equation}
This formulation guarantees $\SO(3)$-equivariant dynamics: $\Delta H_{t+1} \mapsto \Ad_R(\Delta H_{t+1})$. The predicted deformation is projected back to $\mathbb{R}^3$ via the inverse hat map and applied as a per-Gaussian offset to the Gaussian means (positions), following ManiGaussian; covariance and SH attributes are carried over unchanged.

\begin{table*}[t]
\centering
\caption{
\textbf{Multi-task Success Rate (\%) on RLBench.} We evaluate 25 episodes per task for the final checkpoint across 10 challenging tasks (166 variations). Our E3DGS equivariant world-model variant is compared against state-of-the-art baselines. Best results are \textbf{bolded} and second-best results are \underline{underlined}.
}
\label{tab:manigaussian_10tasks}
\resizebox{\textwidth}{!}{
\begin{tabular}{l cccccccccc | c}
\toprule
\textbf{Method / Task} & 
\begin{tabular}{@{}c@{}}close \\ jar\end{tabular} & 
\begin{tabular}{@{}c@{}}open \\ drawer\end{tabular} & 
\begin{tabular}{@{}c@{}}sweep to \\ dustpan\end{tabular} & 
\begin{tabular}{@{}c@{}}meat off \\ grill\end{tabular} & 
\begin{tabular}{@{}c@{}}turn \\ tap\end{tabular} & 
\begin{tabular}{@{}c@{}}slide \\ block\end{tabular} & 
\begin{tabular}{@{}c@{}}put in \\ drawer\end{tabular} & 
\begin{tabular}{@{}c@{}}drag \\ stick\end{tabular} & 
\begin{tabular}{@{}c@{}}push \\ buttons\end{tabular} & 
\begin{tabular}{@{}c@{}}stack \\ blocks\end{tabular} & 
\textbf{Average} \\
\midrule
\textit{Variation Type} & \textit{color} & \textit{placement} & \textit{size} & \textit{category} & \textit{placement} & \textit{color} & \textit{placement} & \textit{color} & \textit{color} & \textit{color, count} & -- \\
\textit{\# Variations} & \textit{20} & \textit{3} & \textit{2} & \textit{2} & \textit{2} & \textit{4} & \textit{3} & \textit{20} & \textit{50} & \textit{60} & \textbf{166} \\
\textit{Keyframes} & \textit{6.0} & \textit{3.0} & \textit{4.6} & \textit{5.0} & \textit{2.0} & \textit{4.7} & \textit{12.0} & \textit{6.0} & \textit{3.8} & \textit{14.6} & -- \\
\midrule
PerAct & 18.7 & 54.7 & 0.0 & 40.0 & 38.7 & 18.7 & 2.7 & 5.3 & {18.7} & \underline{6.7} & 20.4 \\
PerAct (4 cameras) & 21.3 & 44.0 & 0.0 & \underline{65.3} & 46.7 & 16.0 & 6.7 & 12.0 & 9.3 & 5.3 & 22.7 \\
GNFactor & \underline{25.3} & \underline{76.0} & 28.0 & 57.3 & 50.7 & \underline{20.0} & 0.0 & \underline{37.3} & {18.7} & 4.0 & 31.7 \\
ManiGaussian & \textbf{28.0} & \underline{76.0} & \underline{64.0} & 60.0 & \underline{56.0} & \textbf{24.0} & \underline{16.0} & \textbf{92.0} & \underline{20.0} & \textbf{12.0} & \underline{44.8} \\
\midrule
\textbf{Ours (E3DGS)} & 16.0 & \textbf{88.0} & \textbf{80.0} & \textbf{72.0} & \textbf{60.0} & \textbf{24.0} & \textbf{36.0} & \textbf{92.0} & \textbf{32.0} & {4.0} & \textbf{50.4} \\
\bottomrule
\end{tabular}
}
\end{table*}

\noindent
\begin{minipage}[t]{0.54\textwidth}
\paragraph{Impact of Equivariant Dynamics.}
Table~\ref{tab:manigaussian_10tasks} details the RLBench success rates. E3DGS achieves a competitive overall average success rate (50.4\% vs. 44.8\%). E3DGS excels in precise geometric articulation and robust spatial modeling, outperforming ManiGaussian in \emph{put in drawer} (+20.0\%), \emph{sweep to dustpan} (+16.0\%), \emph{open drawer} (+12.0\%), and \emph{meat off grill} (+12.0\%). In these rigid-body settings, exact $\SO(3)$ symmetry prevents overfitting to spurious pose correlations. 
Furthermore, as shown in Figure~\ref{fig:dyna_loss}, this structural regularization significantly accelerates convergence during training steps and yields a lower dynamic loss ($\mathcal{L}_{dyna}$). The network predicts state transitions more accurately without augmentation. While the unconstrained baseline remains stronger on contact-heavy or long-horizon tasks such as \emph{close jar} and \emph{stack blocks}, and ties on \emph{slide block} and \emph{drag stick}, E3DGS improves the average success rate and helps several rigid-body or articulation-heavy tasks. Together with the lower dynamic loss in Fig.~\ref{fig:dyna_loss}, these results suggest that the unified equivariant representation improves pose-consistent dynamics modeling.
\end{minipage}%
\hfill
\begin{minipage}[t]{0.42\textwidth}
    \vspace{0pt} 
    \centering
    \includegraphics[width=\linewidth]{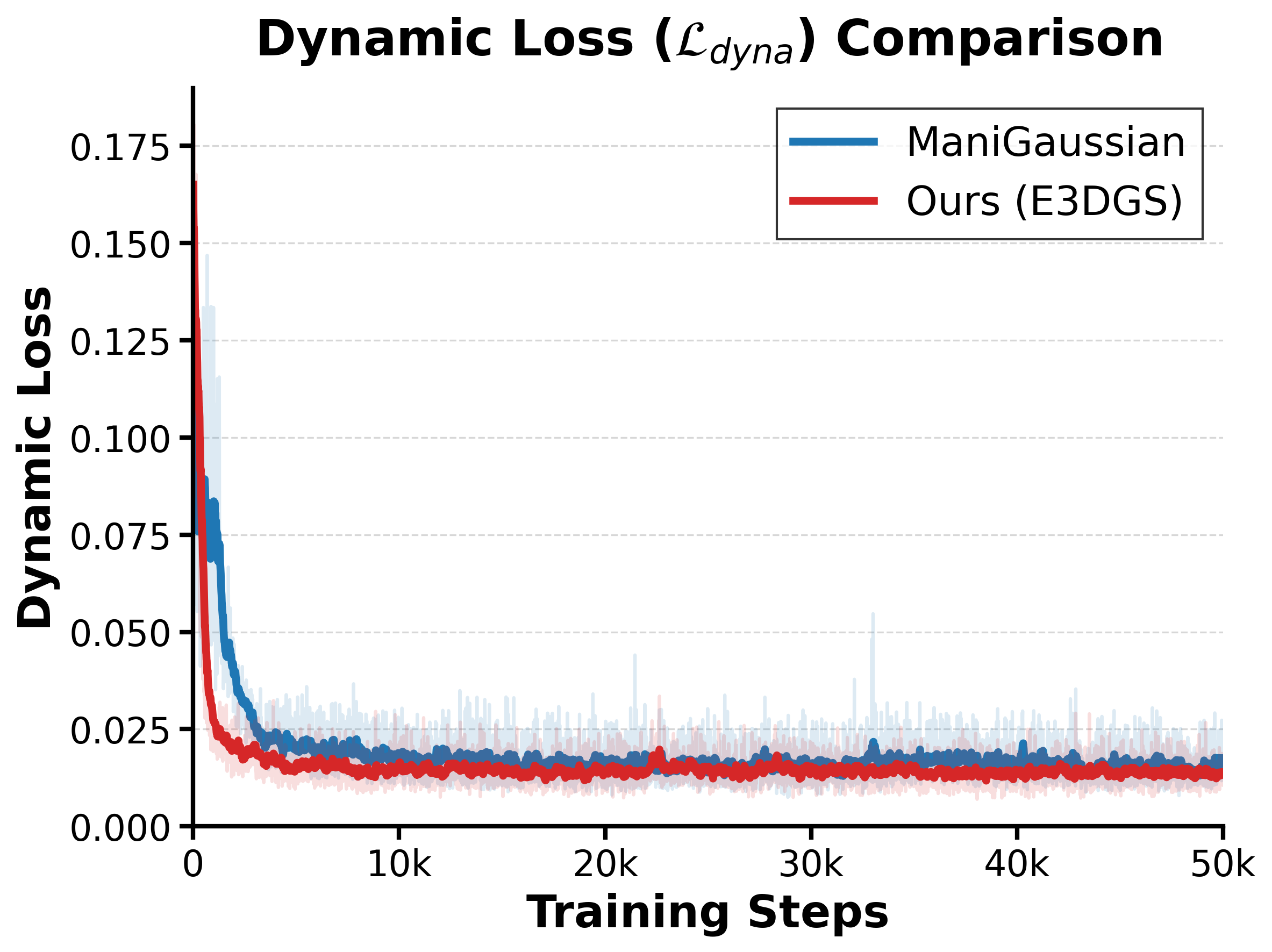}
    \captionof{figure}{\textbf{Dynamic Loss.} E3DGS accelerates the learning of rigid-body state transitions, converging faster to a lower steady-state loss than the baseline.}
    \label{fig:dyna_loss}
\end{minipage}



\section{Conclusion}

We presented \textbf{E3DGS}, a framework introducing rigid-body ($\mathrm{SE}(3)$) equivariance to 3D Gaussian Splatting. At its core is the \emph{Unified Matrix Lifting}, which jointly embeds spatial geometry and view-dependent photometry into a single $\mathfrak{gl}(3)$ carrier. By identifying an isomorphism between Spherical Harmonics ($\ell \le 2$) and matrix conjugation, E3DGS bypasses expensive tensor products, enabling efficient, matrix-native equivariant learning. We validated this formulation via a Gaussian Masked Autoencoder for object recognition and an action-conditioned world model for robotic manipulation. Our results confirm that symmetry-consistent 3DGS learning yields exceptional zero-shot pose robustness, label efficiency, and spatial reasoning.

\paragraph{Scope and outlook.}
Our $\mathfrak{gl}(3)$ formulation is mathematically bounded to SH degrees $\ell\le2$. Although the original 3DGS setting often uses up to degree $3$, SH$(0$--$2)$ captures a substantial photometric subspace and is the maximal bandwidth admitted by the lightweight physical $3\times3$ carrier. Higher bands can be handled by the larger carriers described in Theorem~\ref{thm:general_carrier}, but we leave their empirical evaluation to future work.




%
%
\bibliographystyle{splncs04}
\bibliography{strings-full,ieee-full,main}

\clearpage
\appendix

\section{Explicit $\mathfrak{gl}(3)$ Intertwiners for SH Degrees $\ell \le 2$}
\label{app:sh_intertwiners}

We provide an explicit realization of real spherical harmonics of degrees $\ell=0,1,2$ inside the $3\times 3$ matrix carrier $\mathfrak{gl}(3)$.
Throughout, $R\in\SO(3)$ acts on $\mathfrak{gl}(3)$ by conjugation, $\Ad_R(M):=RMR^\top$.

Under this action,
\begin{equation}
\mathfrak{gl}(3)=\langle I\rangle \oplus \mathfrak{so}(3)\oplus \Sym_0(3),
\label{eq:app_gl3_split}
\end{equation}
where $\langle I\rangle=\{sI:s\in\R\}$, $\mathfrak{so}(3)=\{A:A^\top=-A\}$, and $\Sym_0(3)=\{S:S^\top=S,\ \tr(S)=0\}$.

\subsection{Degree $\ell=0$}
For $f_0\in\R$, define
\begin{equation}
\Phi_0(f_0):=f_0 I.
\end{equation}
Since $D^0(R)=1$ and $\Ad_R(I)=I$, this is an $\SO(3)$-equivariant map.

\subsection{Degree $\ell=1$}
For $f_1\in\R^3$, define
\begin{equation}
\Phi_1(f_1):=\hatop{f_1}\in\mathfrak{so}(3).
\end{equation}
Using $\hatop{Rv}=R\hatop{v}R^\top$ for all $v\in\R^3$, we obtain
\begin{equation}
\Phi_1(D^1(R)f_1)=R\,\Phi_1(f_1)\,R^\top.
\end{equation}

\subsection{Degree $\ell=2$}
Choose the following Frobenius-orthonormal basis of $\Sym_0(3)$:
\begin{align}
B_1 &= \frac{1}{\sqrt{2}}
\begin{bmatrix}
0&1&0\\
1&0&0\\
0&0&0
\end{bmatrix},
&
B_2 &= \frac{1}{\sqrt{2}}
\begin{bmatrix}
0&0&1\\
0&0&0\\
1&0&0
\end{bmatrix},
&
B_3 &= \frac{1}{\sqrt{2}}
\begin{bmatrix}
0&0&0\\
0&0&1\\
0&1&0
\end{bmatrix},\\[0.5em]
B_4 &= \frac{1}{\sqrt{2}}
\begin{bmatrix}
1&0&0\\
0&-1&0\\
0&0&0
\end{bmatrix},
&
B_5 &= \frac{1}{\sqrt{6}}
\begin{bmatrix}
-1&0&0\\
0&-1&0\\
0&0&2
\end{bmatrix}.
\end{align}
Fix a SH basis for degree $2$ whose coefficient vector is identified with the coefficients in this basis.\footnote{Any other orthonormal SH convention differs only by a fixed orthogonal change of basis, so the equivariance statement below is unchanged.}
For $f_2=(a_1,\dots,a_5)^\top\in\R^5$, define
\begin{equation}
\Phi_2(f_2):=\sum_{m=1}^{5} a_m B_m \in \Sym_0(3).
\label{eq:app_phi2}
\end{equation}
The induced action on these coefficients is precisely the Wigner-$D$ action for $\ell=2$, hence
\begin{equation}
\Phi_2(D^2(R)f_2)=R\,\Phi_2(f_2)\,R^\top.
\end{equation}

\paragraph{Implementation convention for 3DGS real SH coefficients.}
We use the real-SH coefficient ordering of the reference 3DGS utilities. For each
color channel, the coefficients are stored as
\begin{equation}
\label{eq:3dgs-sh-order}
g =
\bigl(
g_0
\,\big|\,
g_1,g_2,g_3
\,\big|\,
g_4,g_5,g_6,g_7,g_8
\bigr).
\end{equation}
Here \(g_0\) is the degree-\(0\) coefficient. The degree-\(1\) slots follow the
code order proportional to
\begin{equation}
\label{eq:3dgs-sh-slots}
\begin{aligned}
(g_1,g_2,g_3)
&\longleftrightarrow (-y,\ z,\ -x),\\
(g_4,\ldots,g_8)
&\longleftrightarrow
\bigl(
xy,\ yz,\ 2z^2-x^2-y^2,\ xz,\ x^2-y^2
\bigr).
\end{aligned}
\end{equation}
For compact notation, write \(C_{2,j}:=C_2[j]\). We use the 3DGS real-SH
constants
\begin{equation}
\label{eq:3dgs-sh-constants}
\begin{gathered}
C_1 = 0.4886025,\\
C_{2,0}=1.0925484,\quad
C_{2,1}=-1.0925484,\quad
C_{2,2}=0.3153916,\\
C_{2,3}=-1.0925484,\quad
C_{2,4}=0.5462742 .
\end{gathered}
\end{equation}

With this convention, the degree-\(1\) vector used by \(\Phi_1\) is
\begin{equation}
\label{eq:3dgs-phi1-vector}
u =
\begin{bmatrix}
-C_1 g_3\\
-C_1 g_1\\
\phantom{-}C_1 g_2
\end{bmatrix},
\qquad
\Phi_1(g_1,g_2,g_3)=[u]_\times .
\end{equation}
We use the hat-map convention
\begin{equation}
\label{eq:hat-map-convention}
[u]_\times =
\begin{bmatrix}
0 & -u_z & u_y\\
u_z & 0 & -u_x\\
-u_y & u_x & 0
\end{bmatrix},
\qquad
[u]_\times w = u\times w .
\end{equation}

For degree \(2\), we convert the 3DGS coefficient order to the STF basis
coefficients \(a_1,\ldots,a_5\) used in \(\eqref{eq:app_phi2}\) by matching the
quadratic form \(d^\top S d\):
\begin{equation}
\label{eq:3dgs-stf-coefficients}
\begin{aligned}
a_1 &= \frac{C_{2,0}}{\sqrt{2}}\,g_4,\\
a_2 &= \frac{C_{2,3}}{\sqrt{2}}\,g_7,\\
a_3 &= \frac{C_{2,1}}{\sqrt{2}}\,g_5,\\
a_4 &= \sqrt{2}\,C_{2,4}\,g_8,\\
a_5 &= \sqrt{6}\,C_{2,2}\,g_6 .
\end{aligned}
\end{equation}
Thus
\begin{equation}
\label{eq:3dgs-phi2-implementation}
\Phi_2(g_4,\ldots,g_8)
=
\sum_{m=1}^{5} a_m B_m .
\end{equation}
Equivalently, for \(d=(x,y,z)^\top\) and \(S=\Phi_2(g_4,\ldots,g_8)\), this choice
reproduces the renderer-side degree-\(2\) polynomial
\begin{equation}
\label{eq:3dgs-degree2-quadratic}
\begin{aligned}
d^\top S d
={}&
C_{2,0}g_4\,xy
+
C_{2,1}g_5\,yz
+
C_{2,2}g_6\,(2z^2-x^2-y^2)\\
&+
C_{2,3}g_7\,xz
+
C_{2,4}g_8\,(x^2-y^2).
\end{aligned}
\end{equation}

We numerically verify the full lift using the same real-SH convention for both
coefficient rotation and lifting. With
\(D^{\leq 2}(R)=D^0(R)\oplus D^1(R)\oplus D^2(R)\), sampled rotations satisfy
\begin{equation}
\label{eq:3dgs-intertwiner-check}
\begin{aligned}
\epsilon(R,f)
&:=
\left\|
\Phi_{\leq 2}\!\left(D^{\leq 2}(R)f\right)
-
\operatorname{Ad}_R\!\left(\Phi_{\leq 2}(f)\right)
\right\|_F,\\
\epsilon(R,f)
&\approx 10^{-15}.
\end{aligned}
\end{equation}
If a different SH ordering or normalization is used, the constants and index
permutation should be replaced by the corresponding fixed change-of-basis matrix.
The equivariance statement is unchanged as long as the same convention is used for
both \(D^\ell(R)\) and \(\Phi_\ell\).

\subsection{Combined degree-$\le 2$ lift}
Combining the three cases gives
\begin{equation}
\Phi_{\le 2}:=\Phi_0\oplus\Phi_1\oplus\Phi_2,
\end{equation}
which satisfies
\begin{equation}
\Phi_{\le 2}\!\bigl(D^{\le 2}(R)f_{\le 2}\bigr)
=
R\,\Phi_{\le 2}(f_{\le 2})\,R^\top.
\label{eq:app_phi_le2}
\end{equation}
This is the intertwiner used throughout the paper.

\section{Exact Capacity of the $3 \times 3$ Carrier and Its Alignment with SH(0--2)}
\label{app:sh_capacity}

Let $V_1 = \mathbb{R}^3$ denote the standard irreducible representation of $\SO(3)$.
Under the conjugation action, the space of $3 \times 3$ matrices is identified with the endomorphism ring:
\begin{equation}
\End(V_1) \cong V_1 \otimes V_1 .
\end{equation}
By the Clebsch--Gordan decomposition, this tensor product decomposes into irreducible subrepresentations as:
\begin{equation}
V_1 \otimes V_1 \cong V_0 \oplus V_1 \oplus V_2 .
\end{equation}
Therefore, the $3 \times 3$ conjugation carrier realizes exactly the irreducible types corresponding to spherical harmonics of degrees $\ell=0,1,2$.
Notably, no $V_3$ component appears in this decomposition. Hence, by Schur's Lemma, we have:
\begin{equation}
\Hom_{\SO(3)}(V_3,\End(\mathbb{R}^3)) = \{0\},
\end{equation}
which formally proves that there exists no non-trivial $\SO(3)$-equivariant linear map from degree-3 spherical harmonics into the $3 \times 3$ conjugation carrier.
The $3\times3$ carrier covers SH degrees $0\le \ell\le2$; higher degrees used by practical 3DGS systems are handled by the larger carriers of \cref{thm:general_carrier}.

\section{General Matrix Carrier for Arbitrary SH Degrees}
\label{app:general_carrier}

We show that the $3\times3$ instantiation used in the main paper is the \emph{minimal} case of a general construction: for any finite SH degree, the Wigner-$D$ action is realizable exactly as matrix conjugation on a suitable carrier $\End(V_k)$. This establishes that the SH$(0\text{--}2)$ scope is a property of the minimal carrier, not of the adjoint-matrix principle (\cref{thm:general_carrier}).

\paragraph{Setup.}
Let $V_\ell$ denote the real degree-$\ell$ irreducible representation of $\SO(3)$ (equivalently, degree-$\ell$ real SH coefficients), $\dim V_\ell=2\ell+1$, transforming as $c_\ell\mapsto D^\ell(R)c_\ell$. In a real orthonormal SH basis $D^j(R)$ is orthogonal, hence $V_j^\ast\simeq V_j$ and
\begin{equation}
\End(V_k)\;\simeq\; V_k\otimes V_k^\ast\;\simeq\; V_k\otimes V_k
\;\simeq\;\bigoplus_{\ell=0}^{2k}V_\ell ,
\label{eq:end_vk_decomp}
\end{equation}
by the Clebsch--Gordan rule $V_{j_1}\otimes V_{j_2}\simeq\bigoplus_{\ell=|j_1-j_2|}^{j_1+j_2}V_\ell$~\cite{varshalovich1988angularmomentum,edmonds1996angular}. The carrier used in this paper is $k=1$: $\End(V_1)\simeq\mathfrak{gl}(3)\simeq V_0\oplus V_1\oplus V_2$, with the conjugation action $M\mapsto D^1(R)MD^1(R)^\top=RMR^\top$.

\paragraph{Equivariant injection for any degree.}
By \eqref{eq:end_vk_decomp}, for any finite SH degree $L$ choosing $k\ge\lceil L/2\rceil$ makes $V_L$ appear in $\End(V_k)$. Concretely, one selects an irreducible tensor-operator basis $\{T^{(\ell,k)}_m\}_{m=-\ell}^{\ell}\subset\End(V_k)$ satisfying the standard transformation rule~\cite{varshalovich1988angularmomentum,edmonds1996angular}
\begin{equation}
D^k(R)\,T^{(\ell,k)}_m\,D^k(R)^{-1}=\sum_{m'=-\ell}^{\ell}T^{(\ell,k)}_{m'}\,D^\ell_{m'm}(R),
\label{eq:ito_transform}
\end{equation}
which defines an equivariant injection $\iota_\ell:V_\ell\hookrightarrow\End(V_k)$, $\iota_\ell(c)=\sum_m c_m T^{(\ell,k)}_m$. A direct computation then gives the intertwining property
\begin{equation}
D^k(R)\,\iota_\ell(c)\,D^k(R)^{-1}
=\sum_{m,m'}c_m T^{(\ell,k)}_{m'}D^\ell_{m'm}(R)
=\iota_\ell\!\bigl(D^\ell(R)c\bigr),
\end{equation}
i.e.\ the Wigner-$D^\ell$ action is realized exactly as conjugation on the matrix carrier. This proves \cref{thm:general_carrier}.

\paragraph{Example: SH$3$.} $V_3$ is absent from $\End(V_1)\simeq\mathfrak{gl}(3)$, but
\begin{equation}
\End(V_2)\simeq V_2\otimes V_2\simeq V_0\oplus V_1\oplus V_2\oplus V_3\oplus V_4 ,
\end{equation}
so $V_3$ appears with multiplicity one; the Clebsch--Gordan change of basis yields an equivariant inclusion $\Phi_3:V_3\hookrightarrow\End(V_2)$, unique up to scale, with $\Phi_3(D^3(R)f_3)=D^2(R)\Phi_3(f_3)D^2(R)^{-1}$. Thus SH$3$ is exactly represented by conjugation in a $5\times5$ carrier $D^2(R)$ rather than the physical $3\times3$ carrier $R$. In general SH degrees up to $L$ are handled by $\End(V_{\lceil L/2\rceil})$.

\paragraph{Numerical verification.}
We solved the infinitesimal intertwiner equations as a sanity check. For $X\in\mathfrak{so}(3)$ the induced action on $\End(V_k)$ is $\rho_{\End(V_k)}(X)M=\rho_{V_k}(X)M-M\rho_{V_k}(X)$, and an intertwiner $\Phi:V_3\to\End(V_k)$ must satisfy $\rho_{\End(V_k)}(X)\Phi=\Phi\,\rho_{V_3}(X)$. Solving these linear constraints gives
\begin{equation}
\dim\Hom_{\SO(3)}(V_3,\End(V_1))=0,\qquad
\dim\Hom_{\SO(3)}(V_3,\End(V_2))=1 .
\end{equation}
The resulting $V_3\to\End(V_2)$ map has rank $7$, and a finite-rotation test yields relative equivariance error $1.68\times10^{-15}$, numerically confirming that SH$3$ is impossible in the $3\times3$ carrier but exact in the $5\times5$ adjoint carrier.

\section{Unified Lifting and Symmetry Properties}
\label{app:unified_lifting}

\subsection{Primitive parameterization and centered coordinates}
We parameterize a 3D Gaussian primitive as
\begin{equation}
\mathcal{G}_i=
\Bigl(
\mu_i,\Sigma_i,\alpha_i,\{f_{i,\ell}^{(c)}\}_{\ell=0}^{2,\ c\in\{r,g,b\}}
\Bigr),
\end{equation}
where $\mu_i\in\R^3$, $\Sigma_i\in\SPD(3)$, $\alpha_i\in\R$, and $f_{i,\ell}^{(c)}\in\R^{2\ell+1}$ denotes the real SH coefficients of degree $\ell$ for color channel $c$.

For recognition tasks, global translations are removed by centering:
\begin{equation}
\bar{\mu}_i:=\mu_i-\frac{1}{N}\sum_{j=1}^{N}\mu_j.
\label{eq:app_centered_mu}
\end{equation}
Under a rigid motion $(R,t)\in\SE(3)$, the centered means satisfy $\bar{\mu}_i\mapsto R\bar{\mu}_i$, so the active channels become homogeneous under the rotational part.

\subsection{Geometry lift}
\paragraph{Position.}
We lift the centered position via the hat map,
\begin{equation}
P_i:=\hatop{\bar{\mu}_i}\in\mathfrak{so}(3),
\end{equation}
which satisfies
\begin{equation}
\hatop{R\bar{\mu}_i}=R\hatop{\bar{\mu}_i}R^\top.
\label{eq:app_hat_centered}
\end{equation}

\paragraph{Covariance.}
We lift the covariance through the matrix logarithm,
\begin{equation}
C_i:=\log(\Sigma_i)\in \Sym(3)\subset\mathfrak{gl}(3).
\label{eq:app_cov_lift}
\end{equation}
Since the matrix logarithm commutes with orthogonal congruence,
\begin{equation}
\log(R\Sigma_iR^\top)=R\log(\Sigma_i)R^\top.
\label{eq:app_cov_equiv}
\end{equation}
This preserves exact $\SO(3)$ equivariance while mapping $\SPD(3)$ to the linear space $\Sym(3)$.

\subsection{Photometry lift}
For each color channel $c\in\{r,g,b\}$, define
\begin{equation}
S_i^{(c)}:=\Phi_{\le 2}\!\bigl(f_{i,\le 2}^{(c)}\bigr)\in\mathfrak{gl}(3),
\qquad
f_{i,\le 2}^{(c)}:=\bigl(f_{i,0}^{(c)},f_{i,1}^{(c)},f_{i,2}^{(c)}\bigr).
\label{eq:app_photometry_lift}
\end{equation}
By \eqref{eq:app_phi_le2}, each photometric matrix channel transforms by the same conjugation law as the geometric channels.

\subsection{Equivariant and invariant feature streams}
We collect the active matrix-valued channels as
\begin{equation}
H_i:=
\bigl[
P_i,\ C_i,\ S_i^{(r)},\ S_i^{(g)},\ S_i^{(b)}
\bigr]
\in \mathfrak{gl}(3)^{C_{\mathrm{eq}}}.
\end{equation}
Alongside them, we maintain a separate invariant scalar stream
\begin{equation}
s_i:=
\bigl[
\alpha_i,\ f_{i,0}^{(r)},\ f_{i,0}^{(g)},\ f_{i,0}^{(b)},\ E_{\mathrm{task}}
\bigr]^\top
\in\R^{C_0}.
\label{eq:app_scalar_branch}
\end{equation}
The degree-$0$ appearance coefficients are intentionally duplicated in the scalar branch: although they are already represented inside the photometric lift through $\Phi_0$, the explicit scalar copy gives the network direct invariant access for gating and readout.

\begin{proposition}[Equivariance of the lifted representation]
\label{prop:app_lift_equiv}
For any $R\in\SO(3)$, each active channel of $H_i$ transforms as
\[
(H_i)_k \mapsto \Ad_R\bigl((H_i)_k\bigr)=R(H_i)_kR^\top,
\]
while $s_i$ is invariant.
\end{proposition}

\begin{proof}
The claim follows from \eqref{eq:app_hat_centered}, \eqref{eq:app_cov_equiv}, and \eqref{eq:app_phi_le2}.
\end{proof}

\subsection{Equivariant processing on the unified carrier}
Let $\tilde{B}:\mathfrak{gl}(3)\times\mathfrak{gl}(3)\to\R$ denote the modified Killing form
\begin{equation}
\tilde{B}(X,Y):=6\,\tr(XY)-\tr(X)\tr(Y).
\label{eq:app_modified_killing}
\end{equation}
This bilinear form is invariant under conjugation:
\[
\tilde{B}(\Ad_R X,\Ad_R Y)=\tilde{B}(X,Y).
\]

The E3DGS backbone is built from $\Ad$-equivariant channel mixing, invariant-gated nonlinearities, ReLN-Attention, and ReLN-LayerNorm. Since each primitive commutes with the channel-wise conjugation action and all gating quantities are built from $\SO(3)$-invariant scalars, the backbone preserves exact $\SO(3)$ equivariance on the active channels.

\begin{theorem}[Equivariance of the lifted encoder]
\label{thm:app_encoder_equiv}
Let $\mathcal{F}_\theta(H,s)=(H',s')$ be any encoder composed of $\Ad$-equivariant matrix operations together with invariant scalar modulation. Then
\begin{equation}
\mathcal{F}_\theta(\Ad_R(H),\,s)=\bigl(\Ad_R(H'),\,s'\bigr)
\qquad \forall R\in\SO(3),
\end{equation}
where $\Ad_R$ acts channel-wise on the active matrix stream.
\end{theorem}

\paragraph{Translation handling.}
The exact homogeneous action in the lifted carrier is with respect to the rotational part.
For recognition, centering \eqref{eq:app_centered_mu} yields global translation invariance.
For dynamics, translational quantities are handled through centered or relative vector features before hat lifting, so that the active channels again obey a homogeneous $\SO(3)$ action.

\subsection{Task-specific readouts}
\paragraph{Invariant readout for recognition.}
For classification, the final matrix-valued features are converted to invariant scalars by $\tilde{B}$-based contractions, concatenated with the invariant branch, pooled across the set, and passed to an MLP classifier. This readout is pose-invariant by construction.

\paragraph{Equivariant readout for dynamics.}
For vector-valued predictions, the output matrix channel is projected onto $\mathfrak{so}(3)$ and mapped back to $\R^3$ by the inverse hat map:
\begin{equation}
v_i=\Bigl(\mathrm{skew}\bigl((H_{\mathrm{out}})_{\mathrm{pos}}\bigr)\Bigr)^\vee,
\qquad
\mathrm{skew}(A):=\tfrac12(A-A^\top).
\label{eq:app_inverse_hat}
\end{equation}
Because both $\mathrm{skew}$ and $(\cdot)^\vee$ commute with orthogonal conjugation, this readout is exactly $\SO(3)$ equivariant.

\section{Active and Passive Rotation Viewpoints}
\label{app:se3_viewpoints}

Under an active rotation of the scene, view directions transform as $d\mapsto Rd$, inducing
\begin{equation}
f_{i,\ell}^{(c)}\mapsto D^\ell(R)\,f_{i,\ell}^{(c)}.
\end{equation}
Under a passive change of camera coordinates, directions transform as $d\mapsto R^\top d$, inducing
\begin{equation}
f_{i,\ell}^{(c)}\mapsto D^\ell(R^\top)\,f_{i,\ell}^{(c)}.
\end{equation}
In a orthonormal SH basis, $D^\ell(R)$ is orthogonal, so $D^\ell(R^\top)=D^\ell(R)^{-1}$.
Our matrix carrier accommodates either convention: the lifted channels transform by conjugation with the corresponding rotation argument.


\section{Implementation Details for Gaussian Perception Tasks}
\label{app:3dgs_details}

This section summarizes the implementation details for the Gaussian perception tasks.
We focus on the components that are specific to the proposed formulation: rigid-motion-aware input handling, SH$(0\text{--}2)$ lifting, the unified matrix carrier, and the task heads used for classification and segmentation.

\subsection{Recognition backbone and feature construction}
Starting from Gaussian-MAE, we replace the original Euclidean feature handling with a two-stream representation consisting of:
(i) an active matrix stream carrying geometry and higher-order photometry, and
(ii) an invariant scalar stream carrying opacity, SH0, and auxiliary scalar information.

Each input Gaussian uses SH coefficients up to degree $2$ for every color channel unless otherwise noted. Accordingly, the active photometric stream contains the lifted degree-$1$ and degree-$2$ appearance content. Local grouping, token mixing, and aggregation are performed directly on matrix-valued channels. The scalar stream interacts with the matrix stream through invariant contractions and scalar gating, so the prescribed transformation law of the active channels is preserved throughout the encoder.

\subsection{Task heads}

\paragraph{Classification head.}
For object classification, the final matrix features are converted to invariant scalar descriptors using $\tilde{B}$-based self-contractions and aggregated across tokens by permutation-invariant global pooling. The pooled invariant descriptor is then passed to a standard MLP classifier. We do not use a learnable \texttt{[CLS]} token, since the prediction is intended to depend only on invariant summaries of the lifted set representation.

\paragraph{Segmentation head.}
For part segmentation, we replace the Gaussian token encoder and backbone. The ReLN encoder and transformer produce matrix-valued group features, which are converted to invariant token descriptors through Killing-form self-contractions. These invariant group features are then propagated to point-level features by the standard feature-propagation module and decoded by a point-wise classifier for part labeling.

\subsection{Datasets and Protocol}
\label{app:3dgs_datasets}
To evaluate representation quality, we employ a two-phase learning stage: large-scale unsupervised pretraining followed by task-specific finetuning~\cite{ma2024shapesplat}.

\paragraph{Pretraining: ShapeSplat.}
We utilize the ShapeSplat dataset~\cite{ma2024shapesplat}, a repository of ShapeNet-derived 3D Gaussian assets. For each 3D object, $N=1024$ primitives are extracted using furthest point sampling. The model is trained to reconstruct these primitives under a $60\%$ random masking ratio.

\paragraph{Downstream evaluation on ModelNet10, ModelNet40, and ShapeNet Part.}
The pretrained encoder is finetuned for object classification on ModelNet10 and ModelNet40, and for part segmentation on ShapeNet Part.
We keep the transfer configuration fixed across methods so that differences in performance can be attributed to the representation rather than to changes in the optimization setup.

\subsection{Transformed evaluation protocol}
We report accuracy on the canonical split and on transformed splits.
For recognition, Gaussian means are centered before position lifting, so the active channels are rotation-equivariant and the full pipeline is translation-invariant at the set level.

If the transformed split applies rigid motions, the Gaussian parameters are updated as
\begin{align}
\mu_i' &= R\mu_i+t,\\
\Sigma_i' &= R\Sigma_iR^\top,\\
f_{i,\ell}^{(c)\prime} &= D^\ell(R)\,f_{i,\ell}^{(c)},
\qquad \ell\in\{0,1,2\}.
\end{align}

\paragraph{Evaluation splits.}
We report accuracy on the following settings:
\begin{itemize}
    \item {\textsc{ID/ID}:} canonical training and canonical testing.
    \item {\textsc{ID/$\SE(3)$}:} canonical training and transformed testing.
    \item {\textsc{$\SE(3)$/$\SE(3)$}:} transformed training and transformed testing.
\end{itemize}

\subsection{Pretraining convergence}
\label{app:gaussian_pretraining_note}

Figure~\ref{fig:loss_convergence} compares the masked pretraining dynamics of the baseline Gaussian-MAE and E3DGS on ShapeSplat. Across all tracked losses, E3DGS converges faster in the early stage of training and reaches a lower final error than the baseline. This trend is consistent across rotation, scale, density, Chamfer distance, spherical-harmonic reconstruction, and total training loss. These results support the role of the unified equivariant formulation as an inductive bias for Gaussian representation learning. Once anisotropic geometry and higher-order appearance channels are represented in transformation-compatible carriers, optimization proceeds in a feature space aligned with the symmetry of the data, leading to more stable and consistently lower reconstruction error throughout training.

\begin{figure*}[h]
    \centering
    \includegraphics[width=0.95\linewidth]{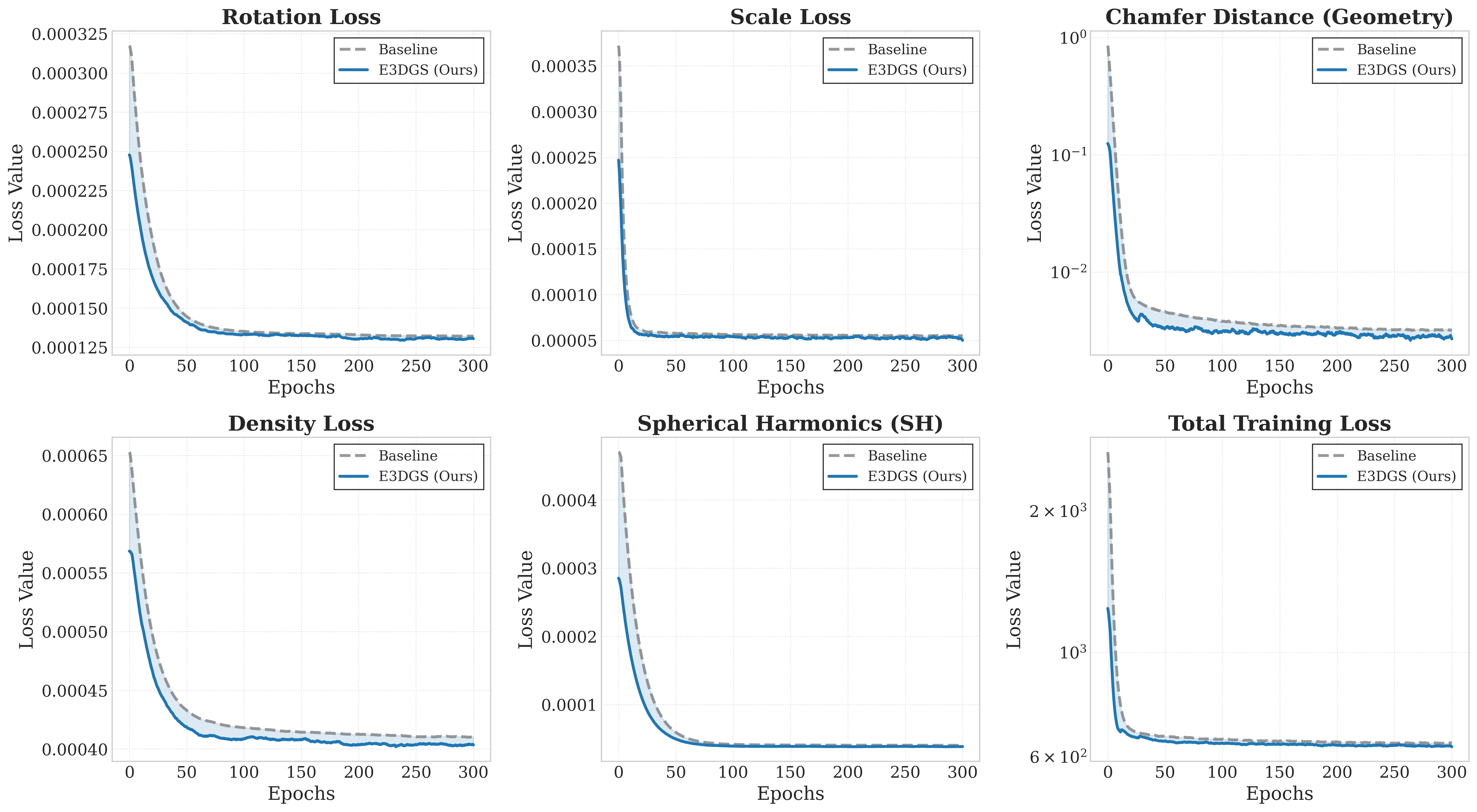}
    \caption{\textbf{Masked pretraining dynamics on 3D Gaussian splats.}
    Comparison of convergence between the baseline Gaussian-MAE~\cite{ma2024shapesplat} and \textbf{E3DGS} over 300 pretraining epochs on ShapeSplat, using Gaussian inputs with SH coefficients up to degree $2$. E3DGS shows faster early convergence and consistently lower loss across all tracked metrics, including rotation, scale, density, Chamfer distance, spherical harmonics, and total training loss.}
    \label{fig:loss_convergence}
\end{figure*}


\paragraph{Network capacity and computational cost.}
To assess whether the comparison is fair, we keep the training protocol matched across models and report both trainable parameter counts and compute statistics in Table~\ref{tab:comp_cost}.
For all entries in the table, both Gaussian-MAE and E3DGS use SH coefficients up to degree $2$, and the models are trained with the same schedule: 300 epochs for pre-training and 300 epochs for downstream fine-tuning on ModelNet40 and ShapeNetPart.

Table~\ref{tab:comp_cost} shows that E3DGS uses substantially fewer trainable parameters than the baseline across all three settings.
For example, in pre-training the trainable parameter count decreases from 29.44M to 12.05M, and similar reductions are observed for classification and part segmentation.
This reduction is achieved by scaling the latent channel width by a factor of $1/3$. The key point is that, after lifting into $\mathfrak{so}(3)$- and $\mathfrak{gl}(3)$-valued carriers, we do not treat the resulting matrix entries as independent scalar channels. Instead, we preserve their geometric matrix structure and normalize the channel budget with respect to the underlying 3D space. Accordingly, we use a width of $C/3$ so that the model capacity remains comparable to the baseline.

At the same time, E3DGS incurs higher arithmetic cost, as reflected by the larger GMAC counts.
This pattern is consistent with the nature of matrix-valued equivariant processing: while the parameter capacity is constrained, the requisite matrix contractions and bilinear operations (e.g., Killing-form evaluations) demand more floating-point operations per parameter than standard scalar baselines.

\begin{table}[t]
\centering
\small
\caption{\textbf{Network capacity and compute comparison.}
Trainable parameter counts and compute cost (GMACs) for the pre-training model and the downstream fine-tuning models (ModelNet40 classification and ShapeNetPart segmentation), measured on an NVIDIA H100 (Grace Hopper GH200) 96GB GPU. Both methods use SH$(0,1,2)$ inputs and the same training schedule.}
\label{tab:comp_cost}
\begin{tabular}{llcc}
\toprule
Task & Model & Trainable Params (M) $\downarrow$ & GMACs \\
\midrule
Pretraining & Gaussian-MAE~\cite{ma2024shapesplat} & 29.44 & 2.05 \\
Pretraining & E3DGS-MAE (Ours) & \textbf{12.05} & 15.39 \\
\midrule
ModelNet40 classification & Gaussian-MAE~\cite{ma2024shapesplat} & 22.10 & 4.79 \\
ModelNet40 classification & E3DGS-MAE (Ours) & \textbf{9.15} & 33.28 \\
\midrule
ShapeNetPart segmentation & Gaussian-MAE~\cite{ma2024shapesplat} & 26.76 & 14.25 \\
ShapeNetPart segmentation & E3DGS-MAE (Ours) & \textbf{10.46} & 42.09 \\
\bottomrule
\end{tabular}
\end{table}


\section{Extended Evaluation on Part Segmentation}
\label{app:partseg_eval}

We further evaluate the proposed representation on ShapeNet Part Segmentation to assess its effect on dense geometric prediction. Compared with object classification, part segmentation places greater emphasis on local structure and therefore provides a complementary test of whether the unified geometric--photometric carrier remains useful when the target is primarily shape-driven. Following the protocol of the baseline~\cite{ma2024shapesplat}, we report class-average mIoU (mIoUC).

\paragraph{Effect of higher-order photometry on segmentation.}
To test the representation under a more demanding photometric setting, we additionally evaluate part segmentation with full SH$(0,1,2)$ inputs, in addition to the SH$(0)$ setting.
Table~\ref{tab:partseg_robust} shows that the effect of higher-order appearance depends strongly on whether its transformation law is handled explicitly.

For Gaussian-MAE, restricting the input to SH$(0)$ yields 82.86 mIoUC on the canonical split and 73.16 under transformed evaluation. When the input is expanded to SH$(0,1,2)$, the canonical score decreases to 77.67 and the transformed score drops to 29.52. Thus, in our reproduced baseline, adding higher-order appearance does not improve segmentation and substantially weakens transformed generalization.

For E3DGS, the trend is different. With SH$(0)$ input, the model attains 80.59 mIoUC on both \textsc{ID/ID} and \textsc{$\SE(3)$/$\SE(3)$}. With SH$(0,1,2)$ input, performance changes only slightly, reaching 80.81 in both settings. Since this comparison does not isolate the individual contributions of degree-$1$ and degree-$2$ channels, we interpret it as evidence that the proposed representation can incorporate higher-order view-dependent appearance without losing pose consistency in dense prediction.

We note a slight underperformance of E3DGS on the canonical \textsc{ID/ID} split compared to the baseline's best SH$(0)$ configuration (80.81 vs. 82.86). This gap highlights a known trade-off in strict equivariant architectures: canonical datasets like ShapeNet contain implicit pose biases (e.g., objects consistently facing a canonical direction), which standard non-equivariant networks can exploit to maximize canonical accuracy.

The part-segmentation results provide a test of whether appearance is represented in a geometry-compatible way. The baseline remains competitive when appearance is restricted to SH$(0)$, where the input contains only invariant color information. Once higher-order SH channels are introduced, however, the same architecture becomes highly sensitive to pose changes. By contrast, E3DGS maintains nearly unchanged performance across transformed evaluation and does not exhibit the same degradation in the SH$(0,1,2)$ setting. These results support the view that the unified carrier is useful not only for global recognition, but also for dense geometric reasoning in the presence of view-dependent photometry.

\begin{table}[h]
\centering
\caption{\textbf{Part segmentation on ShapeNet} (class-average mIoU $\uparrow$[\%]). \textsc{ID} denotes the canonical split and \textsc{$\SE(3)$} denotes the transformed split used in evaluation. The baseline degrades substantially when higher-order SH inputs are introduced, whereas E3DGS maintains stable performance across transformed evaluation and remains robust to SH$(0,1,2)$ input.}
\label{tab:partseg_robust}
\begin{tabular*}{\linewidth}{@{\extracolsep{\fill}}llcc@{}}
\toprule
Method & Input Features & ID / ID & SE(3) / SE(3) \\ \midrule
Gaussian-MAE \cite{ma2024shapesplat} & SH (0) & 82.86 & 73.16 \\
Gaussian-MAE \cite{ma2024shapesplat} & SH (0,1,2) & 77.67 & 29.52 \\ \midrule
{E3DGS-MAE (Ours)} & SH (0) & 80.59 & 80.59 \\
{E3DGS-MAE (Ours)} & SH (0,1,2) & 80.81 & \textbf{80.81} \\ \bottomrule
\end{tabular*}
\end{table}


\section{Action-Conditioned Gaussian World Modeling}
\label{app:manigaussian_details}

We instantiate the manipulation experiments within the ManiGaussian protocol and retain its training and evaluation setup wherever possible. Our modification is focused on the Gaussian state representation and the dynamics module, which are replaced by their equivariant counterparts in the unified matrix carrier.

\subsection{Pipeline Instantiation}

Following ManiGaussian, the manipulation pipeline consists of two coupled parts: (i) a Gaussian world model that predicts the next-step scene from the current observation and robot action, and (ii) a PerceiverIO-based action decoder that predicts the end-effector action from scene features and the language instruction. In the baseline formulation, the world model comprises a representation network $q_\phi$, a Gaussian regressor $g_\phi$, a deformation predictor $p_\phi$, and a Gaussian renderer $\mathcal{R}$; the action decoder then operates on the learned scene representation together with the instruction embedding. In our adaptation, we preserve this overall decomposition. The key difference is that the Gaussian state is first lifted into the unified $\mathfrak{gl}(3)$ carrier, and the state-transition module is replaced with our equivariant dynamics model in the lifted space. The resulting equivariant features are converted to invariant token features before being passed to the PerceiverIO action head, so that the downstream action
interface remains aligned with the ManiGaussian protocol.

\subsection{Training objectives}
We follow the ManiGaussian training design and optimize a weighted sum of action prediction, current-scene reconstruction, future-scene consistency, and semantic distillation losses:
\begin{equation}
\mathcal{L}
=
\mathcal{L}_{\mathrm{Act}}
+
\lambda_{\mathrm{Geo}}\mathcal{L}_{\mathrm{Geo}}
+
\lambda_{\mathrm{Sem}}\mathcal{L}_{\mathrm{Sem}}
+
\lambda_{\mathrm{Dyna}}\mathcal{L}_{\mathrm{Dyna}}.
\end{equation}
Here, $\mathcal{L}_{\mathrm{Act}}$ denotes the behavior-cloning loss on discretized translation, rotation, gripper openness, and collision-avoidance outputs. $\mathcal{L}_{\mathrm{Geo}}$ supervises reconstruction of the current scene from the predicted Gaussian state, and $\mathcal{L}_{\mathrm{Dyna}}$ supervises reconstruction of the future scene conditioned on the current observation and robot action. When semantic supervision is used, $\mathcal{L}_{\mathrm{Sem}}$ matches the projected semantic features of the Gaussian representation to the corresponding foundation-model features.

Relative to the baseline, the loss definitions themselves are unchanged. The modification lies in how the latent Gaussian state and its action-conditioned transition are parameterized: the transition is predicted in the lifted equivariant space and then decoded back to the Gaussian state used for rendering and downstream action prediction.

\subsection{Dataset and Optimization Setup}

We use the same RLBench evaluation protocol as ManiGaussian: 10 language-conditioned manipulation tasks with 166 total task variations. Visual observations are single-front RGB-D images of resolution $128\times128$. For multiview supervision of the Gaussian world model, we use the same number of camera views as the baseline. For a controlled comparison, we keep the optimization recipe unchanged. We use the same PerceiverIO variant as the action decoder, one RTX 4090 GPU, 100k training iterations, batch size 1, the LAMB optimizer with initial learning rate $5\times10^{-4}$, and a cosine schedule with warmup.

\end{document}